\documentclass[10pt,journal,compsoc]{IEEEtran}
\ifCLASSOPTIONcompsoc
 \usepackage[nocompress]{cite}
\else
 \usepackage{cite}
\fi

\ifCLASSINFOpdf
\else
\fi
\ifCLASSOPTIONcompsoc
 \usepackage[font=normalsize,labelfont=sf,textfont=sf]{subfig}
\else
 \usepackage[font=footnotesize]{subfig}
\fi
\ifCLASSOPTIONcompsoc
 \usepackage[font=normalsize,labelfont=sf,textfont=sf]{subfig}
\else
 \usepackage[font=footnotesize]{subfig}
\fi
\usepackage{amssymb}
\usepackage{amsmath}
\usepackage{subfig}
\usepackage{array}
\usepackage{tabularx}
\usepackage{multirow,bigstrut}
\usepackage{booktabs}
\usepackage{graphicx}
\usepackage{mathtools}
\usepackage{algorithmic}
\usepackage{algorithm2e}
\usepackage{listings}
\usepackage{xcolor}
\usepackage{url}
\usepackage{amsthm}
\usepackage{textcomp}
\usepackage{multirow}

\ifCLASSOPTIONcaptionsoff
 \usepackage[nomarkers]{endfloat}
\let\MYoriglatexcaption\caption
\renewcommand{\caption}[2][\relax]{\MYoriglatexcaption[#2]{#2}}
\fi
\usepackage{url}

\theoremstyle{definition}

\DeclareGraphicsExtensions{.pdf}
\graphicspath{{./figures/}}

\makeatletter
\newcommand{\pushright}[1]{\ifmeasuring@#1\else\omit\hfill$\displaystyle#1$\fi\ignorespaces}
\newcommand{\pushleft}[1]{\ifmeasuring@#1\else\omit$\displaystyle#1$\hfill\fi\ignorespaces}
\makeatother

\makeatletter
\IEEEtriggercmd{\reset@font\normalfont\footnotesize}
\makeatother
\IEEEtriggeratref{1}

\newcolumntype{?}[1]{!{\vrule width #1}}

 \def\x{{\mathbf x}}
 
 \def\mn{\ell} 
 \def\mx{u} 
 \def\w{{\mathbf w}} 
 \def\U{{\mathbf U}} 
  
 \def\A{{\mathbf \Lambda}} 
 \def\alphaa {{\boldsymbol \alpha}}

 \newtheorem{proposition}{Proposition}
\begin{document}
\title{Learning Explicit Deep Representations from Deep Kernel Networks}

\author{Mingyuan Jiu,
  Hichem Sahbi
  \thanks{M. Jiu is with School of Information Engineering, Zhengzhou University, Zhengzhou, China. Email: iemyjiu@zzu.edu.cn.} 
\thanks{H. Sahbi is with CNRS, LIP6 UPMC, Sorbonne University, Paris, France. Email: hichem.sahbi@lip6.fr}
}
\maketitle

\begin{abstract}
Deep kernel learning aims at designing nonlinear combinations of multiple standard elementary kernels by training deep networks. This scheme has proven to be effective, but intractable when handling large-scale datasets especially when the depth of the trained networks increases; indeed, the complexity of evaluating these networks scales quadratically \textcolor{black}{w.r.t.} the size of training data and linearly \textcolor{black}{w.r.t.} the depth of the trained networks. \\

\indent  In this paper, we address the issue of efficient computation in Deep Kernel Networks (DKNs) by designing effective maps  in the underlying Reproducing Kernel Hilbert Spaces. Given a pretrained DKN, our method builds its associated Deep Map Network (DMN) whose inner product approximates the original network while being far more efficient. The design principle of our method is greedy and achieved layer-wise, by finding maps that approximate DKNs at different (input, intermediate and output) layers. This design also considers an extra fine-tuning step based on unsupervised learning, that further enhances the generalization ability of the trained DMNs. When plugged into SVMs, these DMNs turn out to be as accurate as the underlying DKNs while being at least an order of magnitude faster \textcolor{black}{on large-scale datasets}, as shown through extensive experiments on the challenging ImageCLEF and COREL5k benchmarks.
\end{abstract}


\IEEEpeerreviewmaketitle
\section{Introduction} \label{sec:intro}

\IEEEPARstart{K}{ernel} design has been an active field of machine learning during the last two decades with many innovative kernel-based algorithms successfully applied to various tasks, including support vector machines (SVMs) for pattern classification and support vector regression for multivariate estimation~\cite{Caputo2004, Lyu2005, Grauman2007,Qi2007,sahbi2002coarse,sahbi2003coarse} as well as kernel-PCA for dimensionality reduction~\cite{Weinberger2004}. The success of these kernel-based algorithms is highly dependent on the choice of kernels; the latter are defined as symmetric and positive semi-definite functions that return similarity between data~\cite{Vapnik1998, ShaweTaylor2004}. Various kernels have been introduced in the literature~\cite{ShaweTaylor2004} including standard elementary kernels (linear, polynomial, Gaussian, histogram intersection, etc.) as well as sophistical ones that model more complex relationships between data \cite{Grauman2007,sahbi2008context,wang2013directed}. However, in practice, knowing a priori which (elementary or sophisticated) kernel is suitable for a given task is not obvious and research has recently been undertaken in order to train suitable kernels for different classification tasks (see for instance~\cite{Lanckriet2004, Yu2009nips, Corinna2010, Sahbi2011, Sahbi2011a,tollari2008comparative,boujemaa2004visual,Sahbi2015}). \\

\indent Among existing solutions, Multiple Kernel Learning (MKL) \cite{Lanckriet2004, Bach2004, Rakotomamonjy2008} has been popular; its principle consists in learning (sparse or convex) linear combinations of elementary kernels that maximize performances for a given classification task. Different MKL algorithms have been proposed in the literature, including constrained quadratic programming~\cite{Lanckriet2004}, second-order cone and semi-infinite linear programming~\cite{Bach2004,Sonnenburg2006} as well as simpleMKL based on mixed-norm regularization~\cite{Rakotomamonjy2008}. In spite of their relative success these solutions hit two major limitations:  on the one hand, the convexity of these simple linear MKL models may limit the space of possible (and also relevant) solutions. On the other hand, MKL solutions, relying on shallow kernel combinations, are less powerful (compared to their deep variants) in order to capture different levels of abstractions in the learned kernel similarity. Considering these two issues, nonlinear and deep architectures have been recently proposed and turned out to be more effective: for instance, hierarchical multiple kernel learning is proposed in~\cite{Bach2009} where elementary kernels are  embedded into acyclic directed graphs while in~\cite{Cortes2009}, nonlinear combination of  polynomial kernels are used. Following the spirit of deep convolutional neural networks~\cite{LeCun98, Krizhevsky2012, Farabet2013}, authors in~\cite{Yu2009nips} adopt kernel functions as a prior knowledge for regularization. In~\cite{Cho2009}, Cho and Saul propose Arc-cosine kernels that  mimic the computation of large neural nets which  can be used in shallow as well as deep networks.  In~\cite{Zhuang2011a}, a multi-layer nonlinear MKL framework is proposed, but it is restricted to only two layers; in this solution, an exponential activation function is applied to each intermediate and output kernel combination. In~\cite{Jiu2015}, Jiu and Sahbi extend this method to a deeper network of more than two layers using a semi-supervised setting that takes into account the topology of training and test data. In all the aforementioned MKL algorithms, the computational complexity of kernel (gram-matrix) evaluation is a major issue that limits the applicability of these methods; indeed, considering a dataset with $N$ samples, this complexity reaches $O(L N^2)$ with $L$ being the depth of the deep kernel networks; this evaluation process is clearly intractable even on reasonable size datasets.\\
\noindent Existing solutions that reduce the computational complexity of evaluating these kernels consider instead explicit maps. In this respect,  different solutions have been proposed in the literature including: the Nystr\"{o}m expansion~\cite{Williams2001} which generates low-rank kernel map approximations of original gram-matrices from uniformly sampled data without replacement\footnote{Bounds, on Nystr\"{o}m approximation and sampling, are given  in \cite{Drineas2005, Kumar2012}}, and the random Fourier sampling (proposed by Rahimi an Recht \cite{Rahimi2007} and extended to group-invariant kernel method in~\cite{LiIonescu2010}) which builds explicit features for stationary kernels using random sampling of the Fourier spectrum. The explicit feature maps for additive homogeneous kernels are also given in~\cite{Vedaldi2012} and  finite approximations are derived based on spectral analysis. Other works have been undertaken including random features \cite{HuangICASSP2013} and convolutional kernel networks \cite{Mairal2014}, which approximate  maps of Gaussians using convolutional neural networks. \\
\indent In this paper, we propose a novel method that reduces the computational complexity of DKN evaluation (and therefore SVM learning) on large datasets. We address the issue of kernel map approximation for {\it any} deep nonlinear combination of elementary kernels rather than one specific type of kernels as achieved in the aforementioned related work. Our solution relies on the {\it positive semi-definiteness} (p.s.d) of existing elementary kernels (linear, polynomial, etc.) and the {\it closure properties} of p.s.d with respect to different operations (including  product, addition and exponentiation) in order to express DKN with DMN. In these closure properties, linear combinations of kernels correspond  to concatenations of their respective maps, while products correspond to Kronecker tensor operations, etc. As some elementary kernels\footnote{such as Gaussian and Histogram Intersection.} used to feed the inputs of DKNs may have infinite dimensional or undefined maps, we consider new explicit maps that accurately approximate these elementary kernels. Considering these maps as inputs, this greedy process  continues layer-wise in order to find all the maps of the subsequent (intermediate and output) layers.  Note that the contribution presented in this paper is an extension of our preliminary work in~\cite{Jiu2016}, but it differs at least in two aspects: first, we consider an unsupervised training criterion that benefits from abundant unlabeled data in order to further decrease the approximation error of the trained DMN and thereby making its generalization power as high as the underlying DKN (and also better than existing elementary and shallow kernel combinations; as shown through experiments). Furthermore, with DMNs, one may employ efficient SVM learning algorithms based on stochastic gradient descent~\cite{Fan2008} on large-scale datasets, rather than usual training algorithms that rely on heavy gram-matrices  and intractable quadratic programming problems. All these statements are corroborated through extensive experiments using two benchmarks: ImageCLEF Photo Annotation~\cite{Villegas2013,sahbi2013cnrs} and COREL5k~\cite{Duygulu2002}. \\

\begin{figure*}[ht]
\begin{center}
\includegraphics[width=0.25\linewidth]{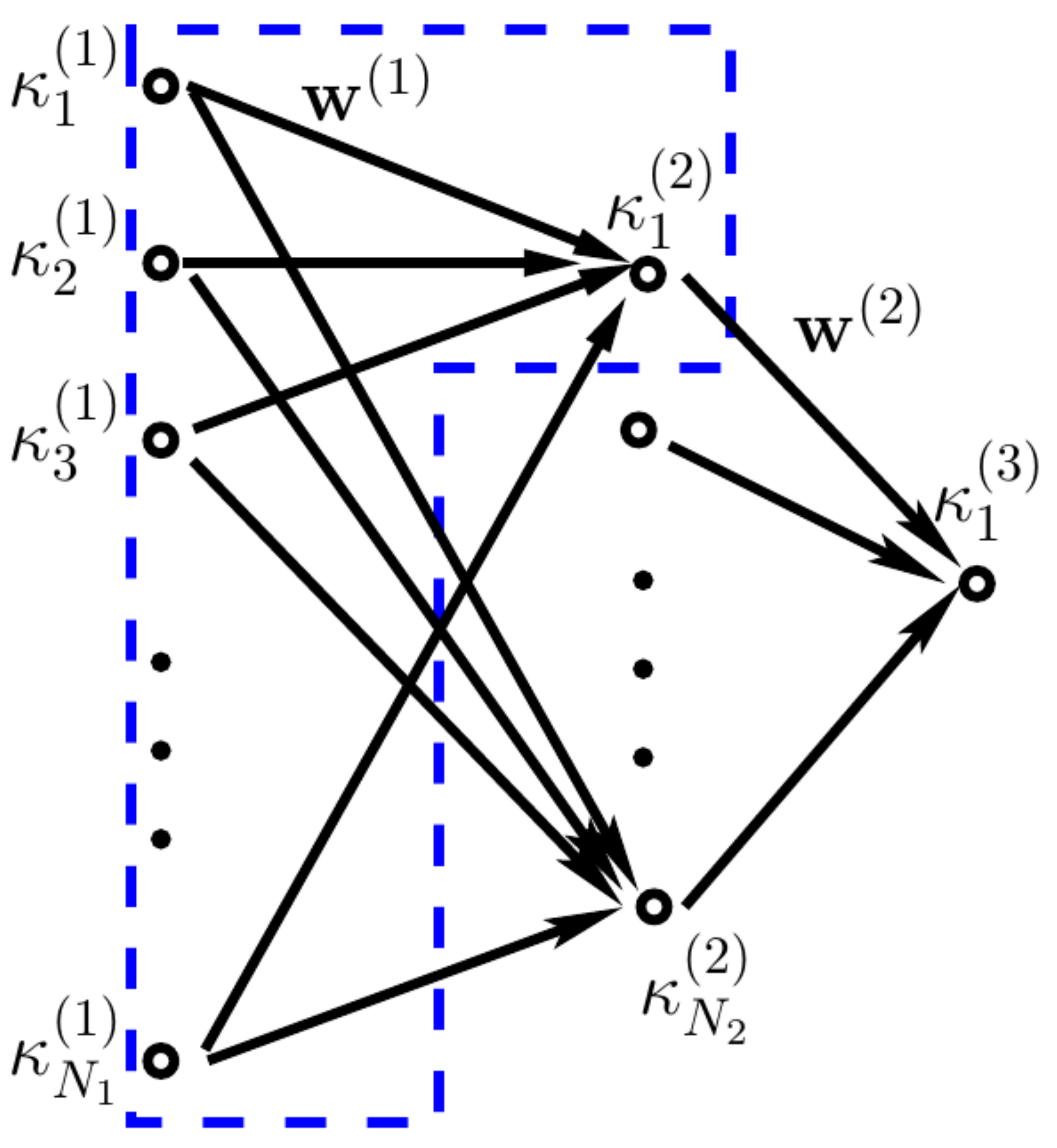}
\hspace{3cm}
\includegraphics[width=0.36\linewidth]{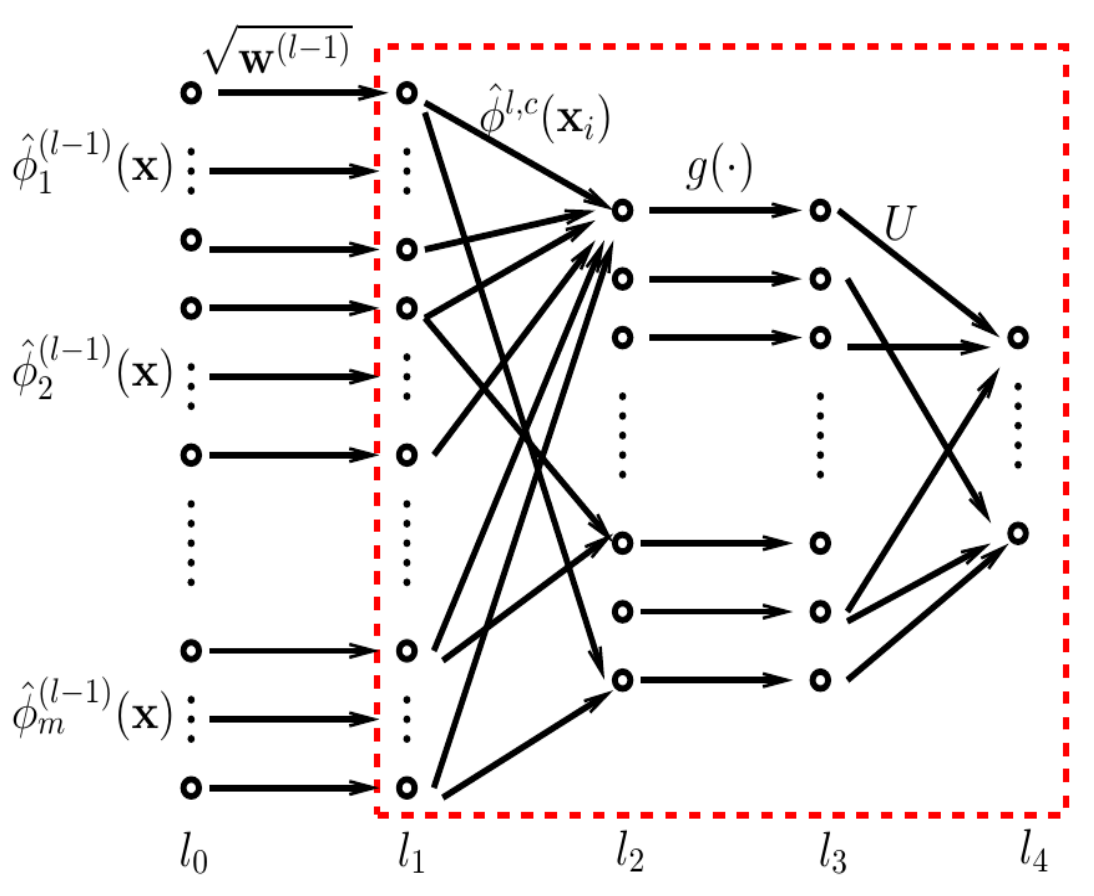}
\end{center}
\caption{\small Left: a three-layer deep kernel network (DKN). Right: a sub-module of deep map network (DMN). The blue dash in the left figure denotes a sub-module of DKN where each node stands for a kernel. The input in the right figure corresponds to the kernel maps and each unit stands for a feature.}
\label{fig:dmn}
\end{figure*}
The rest of this paper is organized as follows: in Section~\ref{sec:dkn} we first briefly remind DKNs, and then in Section~\ref{sec:dmn} we introduce a novel method that builds their equivalent DMNs. In Section~\ref{sec:unsupervisedlearning}, we describe an unsupervised setting of our DMN design while in Section~\ref{sec:experiments}, we present the experimental validation of our method on image annotation tasks using ImageCLEF and COREL5k benchmarks. Finally, we conclude the paper while providing possible extensions for a future work. 

\section{deep kernel networks at a glance} \label{sec:dkn}

A deep kernel network \cite{Zhuang2011a, Jiu2015} is a multi-layered architecture that recursively defines nonlinear combinations of elementary kernels (linear, Gaussian, etc.). Let ${\kappa}_{p}^{(l)}(\cdot,\cdot)$ denote a kernel function  assigned to unit $p$ and layer $l$; ${\kappa}_{p}^{(l)}$ is recursively defined as the output of a nonlinear activation function\footnote{For instance, exponential function \cite{Zhuang2011a}.}  (denoted $g$)  applied to a weighted combination of (input or intermediate) kernels from the preceding layer $(l-1)$ as
\begin{equation}
  {\kappa}_{p}^{(l)}(\cdot,\cdot) = g \big(\sum_q {\bf w}_{p,q}^{(l-1)} \ \kappa_q^{(l-1)}(\cdot,\cdot) \big),
  \end{equation} 
with $\{{\bf w}_{p,q}^{(l-1)}\}$ being weights connecting units at layers $l$ and $l-1$; see the blue dashed area in Fig.~(\ref{fig:dmn}, left). This feed-forward kernel evaluation is achieved layer-wise till reaching the final output kernel. In this recursive definition, other activation functions $g$ can be chosen (particularly for the intermediate layers) including the hyperbolic  making the learning numerically more stable while also preserving the p.s.d of the final output kernel. \\

\noindent For a given classification task, the weights $\{{\bf w}_{p,q}^{(l-1)}\}$ are trained discriminatively \cite{Zhuang2011a, Jiu2015} using a max margin SVM criterion which aims at minimizing a regularized hinge loss on top of the learned DKN. This results into an SVM optimization problem which is solved in its dual form by backpropagating the gradient of that form \textcolor{black}{w.r.t.} the output kernel using the chain rule \cite{LeCun98}, then the weights connecting layers in the DKN are updated using gradient descent. Variants of this optimization criterion, leveraging both labeled and unlabeled data (following a semi-supervised and laplacian setting) makes it possible to train better DKN as detailed in \cite{Jiu2015}. 

\section{Deep map networks} \label{sec:dmn}
In this section, we introduce a novel method that finds for any given DKN, its associated DMN; the proposed method proceeds layer-wise by finding explicit maps that best fit the original kernels in the DKN. As shown later in experiments, this process delivers highly efficient DMNs, while being comparably accurate \textcolor{black}{w.r.t.} their underlying DKNs. Later in Section~\ref{sec:unsupervisedlearning}, we introduce an extension that further enhances the approximation quality of our DMN; starting from the initial weights of the DMN, we update these weights by minimizing the difference between inner products of the maps in the DMN and the original kernels in the DKN. The strength of this extension also resides in its unsupervised setting which makes it possible to learn from abundant unlabeled sets. \\ 
\indent Considering all the elementary (input) kernels in the DKN as positive semi-definite and resulting from the closure of the p.s.d w.r.t. different operations (including sum, product, exponential and hyperbolic activation functions), all the intermediate and output kernels $\{\kappa_{p}^{(l)}\}_{l,p}$ will also be p.s.d. Each $\kappa_{p}^{(l)}(\x,\x')$ can therefore be written as an inner product of kernel maps as $\langle \phi_p^l(\x),\phi_p^l(\x')\rangle$, with  $\phi_p^l: {\cal X} \rightarrow {\cal H}$ being a mapping from the input space $\cal X$ to a high dimensional space $\cal H$.  As the explicit form of $\phi_p^l$ is not necessarily explicit (known), our goal is to design an approximated mapping $\hat{\phi}_p^l$ that guarantees $\kappa_{p}^{(l)}(\x,\x') \simeq \langle \hat{\phi}_p^l(\x), \hat{\phi}_p^l(\x')\rangle$. When these approximated mappings through different layers are known, the resulting DMN provides deep kernel representations from the input data. 
 
\subsection{Input layer maps} \label{sec:inputlayer}

In order to fully benefit from DMNs, the maps of the elementary kernels, that feed these DMNs, should be explicitly known. As discussed earlier, different kernels have different maps; for linear and polynomial, their maps are straightforward and can be easily defined. However, for other more powerful and discriminating kernels, such as the Gaussian and the histogram intersection (HI), their maps are either infinite dimensional or unknown. In this subsection, the definitions of exact and approximate explicit maps are shown for different kernels (including polynomial and HI). \\

\noindent {\bf Exact polynomial kernel map.} An {\it n}-degree polynomial kernel defined as $\kappa_p^{(1)}(\x,\x')=\langle \x,\x'\rangle^n$  can be expressed as $\kappa_p^{(1)}(\x,\x')=\langle \x \otimes^n \x ,\x' \otimes^n \x' \rangle$, with $\otimes^n$ standing for the Kronecker tensor product applied $n$ times. Hence, it is easy to see that the exact explicit map for an {\it n}-degree polynomial kernel is ${\phi}_p^{(1)}(\x) =  \x \otimes^n \x$. \\

\noindent {\bf Approximate HI kernel map.} The approximate explicit maps of HI can be obtained using vector quantization. Given two vectors  $\x$ and $\x'$ of dimension $s$,  the HI on $\x$, $\x'$  is defined as $\kappa_p^{(1)}(\x,\x') =  \sum_{d=1}^s \min(\x^{d},\x'^d)$ (with $\x^{d}$ being the value of $d^{th}$ dimension of $\x$). Considering $\x=(\x^1,\dots,\x^s)^{{\top}} \in {\cal X}$, each dimension $\x^d$ of $\x$ is mapped to 
\begin{equation}
\psi(\x^d) = 2^0 + 2^1 + \dots + 2^{k(\x^d)},   
\end{equation} 
\noindent where $k(\x^d)=\displaystyle \bigg\lfloor Q \frac{\x^d -\mn_d}{\mx_d-\mn_d} \bigg\rfloor$ and $\lfloor z \rfloor$ stands for the largest integer not greater than $z \in \mathbb{R}$, $Q \in \mathbb{N}^+$ is a predefined quantization, ${\mn}_d=\min_{\x} \{ \x^d : \x \in {\cal X}\}$ and ${\mx}_d=\max_{\x} \{ \x^d :  \x \in {\cal X}\}$. In the above definition,  $\psi(.)$ is a ``decimal-to-unary'' map; for instance, $1$ is mapped to $1$, $2$ is mapped to $11$, $3$ to $111$, and so on. In the following, $\psi(\x^d)$ is rewritten as a vector of $Q$ dimensions, and its first $k(\x^d)$ dimensions are set to $1$ and the remaining are set to $0$. 

\begin{proposition} 
Given any $\x$, $\x'$ in $\cal X$, for sufficiently large $Q$, the inner product $\langle \hat{\phi}_p^{(1)}(\x),\hat{\phi}_p^{(1)}(\x')\rangle$ approximates the histogram intersection kernel $\kappa_p^{(1)}(\x,\x')$, where 
{
  \small 
\begin{equation} 
\hat{\phi}_p^{(1)}(\x) = \left(\psi(\x^1)^{{\top}} \sqrt{\frac{\mx_1-\mn_1}{Q}}, \sqrt{\mn_1}, \dots, \psi(\x^s)^{{\top}} \sqrt{\frac{\mx_s-\mn_s}{Q}},\sqrt{\mn_s}\right)^{{\top}}\label{map0}
\end{equation}
}
\noindent is the approximate kernel map and $\psi(\x^d)^\top$ stands for the transpose of $\psi(\x^d)$.

\end{proposition} 
\begin{proof}
$\forall i, j \in \{1,\dots,m\}$, 
\begin{equation}
\begin{array}{lll}
 \langle \hat{\phi}_p^{(1)}(\x_i),\hat{\phi}_p^{(1)}(\x_j) \rangle &  = &  \langle \psi(\x_i^1),\psi(\x_j^1)\rangle (\frac{\mx_1-\mn_1}{Q}) +   \mn_1 +  \dots  \\ 
    &  + &   \langle \psi(\x_i^s),\psi(\x_j^s)\rangle (\frac{\mx_s-\mn_s}{Q}) +   \mn_s 
\end{array}\label{eq1111} 
\end{equation} 
It is easy to see that $\forall  d \in \{1,\dots,s\}$,  $\langle \psi(\x_i^d),\psi(\x_j^d)\rangle = \min(k(\x_i^d),k(\x_j^d))$. By replacing in Eq.~(\ref{eq1111})
\begin{equation}
\small
\begin{array}{lll}
& \bigg| \langle  \hat{\phi}_p^{(1)}(\x_i),\hat{\phi}_p^{(1)}(\x_j) \rangle - \kappa_p^{(1)}(\x_i,\x_j)\bigg|   &  \\ 
      = & \displaystyle  \bigg| \sum_{d=1}^s \min(k(\x_i^d),k(\x_j^d)) \frac{\mx_d-\mn_d}{Q} +   \mn_d - \min(\x_i^d,\x_j^d)\bigg|  &  \\ 
      \leq   & \displaystyle\sum_{d=1}^s \bigg| \bigg\lfloor Q \frac{\min(\x_i^d,\x_j^d) -\mn_d}{\mx_d-\mn_d} \bigg\rfloor \frac{\mx_d-\mn_d}{Q} +   \mn_d  - \min(\x_i^d,\x_j^d)\bigg| &  \\ 
      =  & \displaystyle\sum_{d=1}^s \frac{\mx_d-\mn_d}{Q} \bigg| \bigg\lfloor Q \frac{\min(\x_i^d,\x_j^d) -\mn_d}{\mx_d-\mn_d} \bigg\rfloor -   Q \frac{\min(\x_i^d,\x_j^d) -\mn_d}{\mx_d-\mn_d}  \bigg| &   \\
      \leq  &  \displaystyle \frac{1}{Q}\sum_{d=1}^s \mx_d-\mn_d, \ \   (\textrm{as} \ \  | \lfloor z \rfloor - z | \leq 1), 
\end{array} 
\end{equation} 
as $Q$ increases, $\bigg| \langle \hat{\phi}_p^{(1)}(\x_i),\hat{\phi}_p^{(1)}(\x_j) \rangle - \kappa_p^{(1)}(\x_i,\x_j)\bigg| \leadsto 0$ 
\end{proof} 

\noindent {\bf Approximate Gaussian kernel map.} As the explicit map of the Gaussian kernel is  infinite dimensional, we consider instead an approximate explicit map of  that kernel using {eigen decomposition (ED)} as shown in Eqs.~(\ref{equa:projection}), (\ref{equ:eigenvalueproblem}) with $l=1$ (see Section~\ref{sec:intermediatelayer}). This ED is not restricted to the Gaussian kernel and can also be extended to other kernels whose exact explicit maps are difficult to obtain.

\subsection{Intermediate/output layer maps} \label{sec:intermediatelayer}
 
\noindent Given the explicit map of each elementary kernel at the input layer, our goal is to design the maps of the subsequent layers. Since the map of each layer depends on its preceding layers, this goal is achieved layer-wise  using a greedy process. As intermediate/output kernels in the DKN are defined as linear combinations of kernels in the preceding (input or intermediate) layers followed by nonlinear activations, we mainly focus on how to approximate the maps of these activation functions in the DMN; in this section, we assume that weights $\{{\bf w}_{p,q}^{(l)}\}$ connecting different layers are already known resulting from the initial setting of the DKN (see again Section~\ref{sec:dkn}).

\begin{proposition}

  Let ${\cal S}=\{\x_i\}_{i=1}^N$ be a subset of $N$ samples of $\cal X$, and let ${\bf K}_p^{l}$ be a gram-matrix whose entries are defined on $\cal S$. Let $\U_p^{(l)}=\alphaa \A^{-1\slash2}$ with $\alphaa$, $\boldsymbol{\Lambda}$ being respectively the matrices of eigenvectors and eigenvalues obtained by solving
\begin{equation}
{\bf K}_p^l {\alphaa}  = {\alphaa} \A, 
\label{equ:eigenvalueproblem}
\end{equation}
Considering $\|.\|_2$ as the $\ell_2$ (matrix) norm and $\hat{\bf K}_p^l$ as the gram-matrix associated to $\{\langle \hat{\phi}_p^{(l)}(\x), \hat{\phi}_p^{(l)}(\x')\rangle\}_{\x,\x' \in {\cal S}}$ with 
\begin{equation}
\small 
  \hat{\phi}_p^{(l)}(\x)^{{\top}}= \bigg(g(\langle \hat{\phi}_p^{l,c}(\x), \hat{\phi}_p^{l,c}(\x_1) \rangle) \dots g(\langle \hat{\phi}_p^{l,c}(\x), \hat{\phi}_p^{l,c}(\x_N) \rangle)\bigg) \U_p^{(l)} 
\label{equa:projection}
\end{equation}
and 
\begin{equation}
  \small 
\hat{\phi}_p^{l,c}(\x) =\bigg(\sqrt{\w_{{p,1}}^{(l-1)}} \hat{\phi}_1^{(l-1)}(\x)^{{\top}} \cdots \sqrt{{\w^{(l-1)}_{p,n_{l-1}}}} \hat{\phi}_{n_{l-1}}^{(l-1)}(\x)^{{\top}}\bigg)^{{\top}},
\label{equ:fullfeature}
\end{equation}
\noindent  then the following property is satisfied \begin{equation}\big\|\hat{\bf K}_p^l - {\bf K}_p^l\big\|_2=0.\end{equation} 
\end{proposition}

\begin{proof}
  Let's proceed layer-wise by induction; for $l=1$ (and following section~\ref{sec:inputlayer}), the initial kernel maps  $\{\hat{\phi}_p^{(1)}(.)\}$ are designed to satisfy   $\hat{\phi}_p^{(1)}(\x)^{{\top}} \hat{\phi}_p^{(1)}(\x')=\kappa_p^{(1)}(\x,\x')$. \\
  
\noindent  Now provided that $\hat{\phi}_q^{(l-1)}(\x)^{{\top}} \hat{\phi}_q^{(l-1)}(\x')=\kappa_q^{(l-1)}(\x,\x')$, the property to show is $\hat{\phi}_p^{(l)}(\x)^{{\top}} \hat{\phi}_p^{(l)}(\x')=\kappa_p^{(l)}(\x,\x')$,  $\forall \x \in {\cal S}$.  Following (\ref{equ:fullfeature}) we have
\begin{equation}
  \begin{array}{lll}
    \displaystyle \langle \hat{\phi}_p^{l,c}(\x),\hat{\phi}_p^{l,c}(\x') \rangle &=& \displaystyle  \sum_{q=1}^{n_{l-1}} \w_{p,q}^{(l-1)}  \ \hat{\phi}_q^{(l-1)}(\x)^{{\top}} \hat{\phi}_q^{(l-1)}(\x') \\
                                                                       & = &   \displaystyle  \sum_{q=1}^{n_{l-1}} \w_{p,q}^{(l-1)}  \ \kappa_q^{(l-1)}(\x,\x'),                                                                        
\end{array}\label{eq:00}
    \end{equation} 

    \noindent the second equality results from the hypothesis of induction. By plugging (\ref{eq:00}) into (\ref{equa:projection}), we obtain
    \begin{equation}
\hat{\phi}_p^{(l)}(\x)^{{\top}} =  \big(\kappa_p^{(l)}(\x,\x_1),\dots,\kappa_p^{(l)}(\x,\x_N)\big) \U_p^{(l)}, 
\label{equa:projection2}
\end{equation}
and equivalently  $\hat{\bf K}_p^l={\bf K}_p^l \ \U_p^{(l)} \ \U_p^{(l)\top} \ {\bf K}_p^l$. Hence, 
\begin{equation} 
\begin{array}{lll}
  \big\|\hat{\bf K}_p^l - {\bf K}_p^l\big\|_2 & = &  \big\|{\bf K}_p^l \ \alphaa \A^{-1\slash2} \  \A^{-1\slash2} \ \alphaa^{\top} \ {\bf K}_p^l -{\bf K}_p^l\big\|_2  \\
   & & \\
                 & = & \big\|{\alphaa} \A \  \A^{-1} \ \alphaa^{\top} \ {\bf K}_p^l - {\bf K}_p^l\big\|_2  = 0,                 
\end{array}
\end{equation} 
which also results from Eq.~(\ref{equ:eigenvalueproblem}) and the orthogonality of eigenvectors in $\alphaa$.
\end{proof} 
 
Note that for any samples $\x$, $\x'$ taken out of $\cal S$ (but with similar distribution as $\cal S$), it is clear (as also observed in our experiments) that $|\langle \hat{\phi}_p^{(l)}(\x), \hat{\phi}_p^{(l)}(\x')\rangle - \kappa_p^{(l)}(\x,\x')| \leadsto 0$ as $N$ and the number of eigenvectors used in $\{\U_p^{(l)}\}$ increase.  

\subsection{Network design} \label{sec:networkdesign}

We incrementally expand each layer $l$  in the DKN into three sub-layers in the underlying DMN in order to design the map $\hat{\phi}_p^{(l)}$.  The first sub-layer provides the products between weights $\{(\w_{p,q}^{(l-1)})^{1\slash2}\}_q$ and the preceding maps $\{\hat{\phi}_q^{(l-1)}(\x)\}_q$ resulting into the intermediate map  $\hat{\phi}_p^{l,c}(\x)$ as shown in Eq.~(\ref{equ:fullfeature}). Afterwards, we feed this map $\hat{\phi}_p^{l,c}(\x)$ to Eq.~(\ref{equa:projection}) in two steps: (i) in the second sub-layer, inner products are achieved between $\hat{\phi}_p^{l,c}(\x)$ and parameters  $\{\hat{\phi}_p^{l,c}(\x_i)\}_{i=1}^N$ followed by the activations $\{g(.)\}_{i=1}^N$ (with $g$ being the hyperbolic excepting the final layer in the DKN which uses the exponential); (ii) in the third sub-layer, the explicit map   $\hat{\phi}_p^{(l)}$ is obtained as the product of $\{g(.)\}_{i=1}^N$ and weights $\U_p^{(l)}$. Fig. (\ref{fig:dmn}, right) shows these three sub-layers in the DMN. Similarly, all the subsequent layers in the DMN are designed by processing the DKN layer-wise.\footnote{As the goal, in this paper, is to build approximate deep kernel maps for a given (fixed) deep kernel network, the weights $\mathbf{w}$ between different layers  remain fixed (as shown in Eq.~(\ref{equ:fullfeature})). However, they can also be jointly learned using gradient descent, but this is out of the main scope of this paper.}

\section{Enhancing DMN Parameters} \label{sec:unsupervisedlearning}
So far the design principle of our method (shown in Section \ref{sec:networkdesign} and Fig.~\ref{fig:dmn}) seeks to find explicit maps whose inner products approximate the original kernel values. This is achieved by expanding each layer in the DKN into three sub-layers in the DMN with parameters fixed to $\{\hat{\phi}_p^{l,c}(\x_i)\}_i$ and $\U_p^{(l)}$. In spite of being efficient and also effective w.r.t. the DKN (see experiments), the resulting DMN can be further improved when re-training and fine-tuning these parameters as shown subsequently.

\def\p{{\bf p}}
\def\q {{\bf q}}

The purpose of the proposed unsupervised algorithm is to further reduce the approximation error between the kernel values from DKN and the inner product of kernel maps from DMN. Let ${\cal S}' \subset {\cal X}$ be a subset drawn from the same distribution as $\cal S$ and define $\mathcal{P}$ as a subset of pairs taken from ${\cal S}' \times {\cal S}'$. Our goal is to optimize maps of DMN using the following unsupervised criterion
\begin{equation}
 E = \sum_{(\x, \x') \in {\cal P}} \frac{1}{2} \big\| \hat{\phi}_1^{(L)} (\x)^{\top} \hat{\phi}_1^{(L)}(\x') - \kappa_1^{(L)}(\x, \x')\big\|^2,
 \label{equa:lossfunction}
\end{equation}
here $\kappa_1^{(L)}(\x,\x')$  corresponds to the kernel value obtained using the DKN and $\hat{\phi}_1^{(L)}(\x)$, $\hat{\phi}_1^{(L)}(\x')$ 
are the underlying (unknown) kernel maps;  initially, only  $\{\hat{\phi}_p^{(1)}(\x), \hat{\phi}_p^{(1)}(\x')\}_{(\x,\x') \in {\cal P}}$ are known according to the procedure shown in Section~\ref{sec:inputlayer}.

Considering the initial setting of DMN parameters (i.e., $\{\hat{\phi}_p^{l,c}(\x_i)\}_i$ and $\U_p^{(l)}$), the learning process of this DMN relies on backpropagation~\cite{LeCun98}. The latter finds the best parameters by minimizing the objective function  ($E$) following  an ``end-to-end'' framework where the gradients of $E$ are given using the chain rule;  we firstly compute the gradients of the loss function $E$ w.r.t. final kernel maps, then we backpropagate them through the DMN in order to obtain the gradients w.r.t. the parameters of DMN, finally we average them over training pairs to obtain the descent direction and  update DMN parameters. 

Starting from the derivative of $E$ w.r.t. $\hat{\phi}_1^{(L)}(\x)$
\begin{equation}
 \frac{\partial E}{\hat{\phi}_1^{(L)}(\x)} =  \big(  \hat{\phi}_1^{(L)}(\x)^{\top} \hat{\phi}_1^{(L)}(\x') - \kappa_1^{(L)}(\x,\x') \big) \hat{\phi}_1^{(L)}(\x'),
 \label{equa:gradientfinallayer}
\end{equation}
\noindent we obtain the gradients w.r.t. different layers $l=L,\dots,1$ and units $p=1,\dots,n_l$ as shown in the following section.

\subsection{Error backpropagation} \label{sec:unsuperviseddetails}
As the construction of DMN is achieved layer-wise  (see again Section~\ref{sec:networkdesign}), we show below the backpropagation procedure for a module (shown in Fig.~\ref{fig:dmn}, right). Given the derivatives of $E$ w.r.t. $\hat{\phi}_p^{(l)}(\x)$ in layer $l$, we evaluate  the derivatives w.r.t. $\hat{\phi}_q^{(l-1)}(\x)$ in layer $(l-1)$. The derivative w.r.t. $\hat{\phi}_p^{(l)}(\x)$ is  backpropagated to $\kappa_p^{(l)}$ in Eq.~(\ref{equa:projection}) by
\textcolor{black}{
\begin{equation}
\frac{\partial E}{\partial \kappa_p^{(l)}(\x,\x_i)} = ( \frac{\partial E}{\partial \hat{\phi}_p^{(l)}(\x)} )^{\top}   [\U_p^{(l)}]_i^{\top},
\label{equa:gradientkernel}
\end{equation}}

\noindent here $[.]_i$ stands for the ith row of a matrix. Considering  $\kappa_p^{(l)}(\x,\x_i) = g \big(f_p^{(l)}(\x,\x_i) \big)$,
with $f_p^{(l)}(\x,\x_i)  =  \langle \hat{\phi}_p^{l,c}(\x), \hat{\phi}_p^{l,c}(\x_i) \rangle$, we obtain 
\begin{equation}
 \frac{\partial E}{\partial f_p^{(l)}(\x,\x_i)} = g'(f_p^{(l)}(\x,\x_i)) \ \frac{\partial E}{\partial \kappa_p^{(l)}(\x,\x_i)},
\end{equation}
\noindent where $g'(\cdot)$ is the derivative of the nonlinear activation function; for instance, $g'(\cdot) = 1-\textrm{tanh}(\cdot)^2$ for the tangent hyperbolic  and $g'(\cdot)=g(\cdot)$ for the exponential. By accumulating the derivatives from each term $f_p^{(l)}(\x, \x_i)$, we obtain
\begin{equation}	
\frac{\partial E}{\hat{\phi}_p^{l,c}(\x)} =  \sum_{i=1}^N \hat{\phi}_p^{l,c}(\x_i) \frac{\partial E}{\partial f_p^{(l)}(\x,\x_i)},
\end{equation}
Finally, we get the derivatives w.r.t. $\hat{\phi}_q^{(l-1)}(\x)$ for layer $(l-1)$ in Eq.~(\ref{equ:fullfeature}) by
\begin{equation}
 \frac{\partial E}{\partial \hat{\phi}_q^{(l-1)}(\x)} = \sqrt{\w_{p,q}^{(l-1)}} \textrm{Frag}( \frac{\partial E}{\hat{\phi}_p^{l,c}(\x)} )_q,
 \label{equa:gradientprevisoulayer}
\end{equation}
\noindent where $\textrm{Frag}( \frac{\partial E}{\hat{\phi}_p^{l,c}(\x)} )_q$ stands for the fragment of derivatives corresponding to the kernel maps of the unit $q$ at layer $(l-1)$ in the DKN. 

The gradients of the loss function $E$ w.r.t. $\U_p^{(l)}$ and $\hat{\phi}_p^{l,c}(\x_i)$ are then given as

\textcolor{black}{
\begin{equation}	
  \Delta \U_p^{(l)} = (\kappa_p^{(l)}(\x,\x_1)\dots \kappa_p^{(l)}(\x,\x_N))^{\top} \  ( \frac{\partial E}{\partial \hat{\phi}_p^{(l)}(\x)} )^{\top}
\label{equa:gradientevd}
\end{equation}
}

\begin{equation}
 \Delta \hat{\phi}_p^{l,c}(\x_i) = \frac{\partial E}{\partial f_p^{(l)}(\x, \x_i)} \hat{\phi}_p^{l,c}(\x).
 \label{equa:gradientexamples}
\end{equation}

Error backpropagation is achieved layer-wise  from the final to the input layer; the increments of $\{\hat{\phi}_p^{l,c}(\x_i)\}_{i=1}^N$ and $\U_p^{(l)}$ are obtained by Eq.~(\ref{equa:gradientexamples}) and Eq.~(\ref{equa:gradientevd}). Gradient descent with a step $\eta$ (see experiments) is performed to update the parameters of DMN. The whole learning procedure is shown in Algorithm~\ref{algo:unsupervisedlearning}.

As described earlier, an initial DMN is firstly  set using the training set $\mathcal{S}$, then  sample pairs in  $\mathcal{P}$ are randomly selected from $\mathcal{S}'$ to further enhance the parameters of the new (fine-tuned) DMN.  As a result, the fine-tuned DMN enables us to obtain a better approximation of the original DKN on large datasets while being highly efficient as shown through the following experiments in \textcolor{black}{image annotation}. 
 
\begin{algorithm}[t]
\small
\caption{Unsupervised DMN learning algorithm}
\label{algo:unsupervisedlearning}
\BlankLine
\KwIn{Fixed $\{\mathbf{w}_{p,q}^{(l-1)}\} \  (l=2, \ldots, L)$, \\
A set of sample pairs $\mathcal{P}$, \\
Kernel maps $\{\hat{\phi}_p^{(1)}(\x)\}_{\x \in {\cal S}'}$ at the input layer, \\
Output kernel values $\{ \kappa_1^{(L)}(\x,\x')\}_{(\x,\x') \in {\cal P}}$. \\
Initialization: $\{\hat{\phi}_p^{l,c}(\x_i)\}_{i=1}^N$ and $\{\U_p^{(l)}\}$, $p=\{1, \ldots, n_l\}$, learning rate $\eta$.}
\KwOut{Optimal (updated) $\{\hat{\phi}_p^{l,c}(\x_i)\}_{i=1}^N$ and $\{\U_p^{(l)}\}$. } 

\Repeat{Convergence}
{
    \For{ each pair $(\x,\x') \in {\cal P}$}
    {Forward $(\hat{\phi}_p^{(1)}(\x), \hat{\phi}_p^{(1)}(\x'))$ through DMN to obtain $(\hat{\phi}_p^{(L)}(\x), \hat{\phi}_p^{(L)}(\x'))$by Eqs.~(\ref{equ:fullfeature}), (\ref{equa:projection})\;
    Compute the loss by Eq.~(\ref{equa:lossfunction})\;
    Compute the gradients $\frac{\partial E}{\hat{\phi}_1^{(L)}(\x)}$ by Eq.~(\ref{equa:gradientfinallayer})\;
    
	\For{$l=L:2$}
	{
	Backward the gradients $\frac{\partial E}{\hat{\phi}_1^{(L)}(\x)}$ by Eqs.~(\ref{equa:gradientkernel})-(\ref{equa:gradientprevisoulayer})\;
	
	Compute $(\Delta U_p^{(l)})_{\x}$ and $(\Delta \hat{\phi}_p^{l,c}(\x_i))_{\x}$ by Eq.~(\ref{equa:gradientevd}) and (\ref{equa:gradientexamples})\;  
	
        Compute the gradients from $\hat{\phi}_p^{(l)}(\x')$: $(\Delta U_p^{(l)})_{\x'}$ and $(\Delta \hat{\phi}_p^{l,c}(\x_i))_{\x'}$\;	
    
        Average both gradients: $\Delta U_p^{(l)}\leftarrow \frac{1}{2}\big((\Delta U_p^{(l)})_{\x}+(\Delta U_p^{(l)})_{\x'}\big)$\;
        $\Delta \hat{\phi}_p^{l,c}(\x_i) \leftarrow \frac{1}{2}\big((\Delta \hat{\phi}_p^{l,c}(\x_i))_{\x}+(\Delta \hat{\phi}_p^{l,c}(\x_i))_{\x'}\big)$;
    
        Update these parameters by gradient descents\;
        $ U_p^{(l)} \leftarrow  U_p^{(l)} - \eta \Delta U_p^{(l)}$ \;
        $\hat{\phi}_p^{l,c}(\x_i) \leftarrow \hat{\phi}_p^{l,c}(\x_i) - \eta \Delta \hat{\phi}_p^{l,c}(\x_i)$\;	
	}
    }
   
}
\end{algorithm}

\section{Experiments} \label{sec:experiments}
In this section, we compare the performance of the proposed DMN w.r.t. its underlying DKN in three aspects: i) discrimination power, ii) relative approximation error between DMN and DKN and iii) also efficiency.  The targeted task is image annotation (e.g.,~\cite{sahbi2013cnrs,li2011superpixel}); given a picture, the goal is to predict a list of keywords that best describes the visual content of that image. We consider two challenging and widely used benchmarks:  ImageCLEF~\cite{Villegas2013} and COREL5k~\cite{Duygulu2002} (see details below). For both sets, we learn -- highly competitive -- 3-layer DKNs using the setting in~\cite{Jiu2015} and we plug these DKNs into SVMs in order to achieve image classification and annotation. 

The discrimination power of the learned DMN and DKN networks is measured following the protocol defined by challenge organizers and data providers (see \cite{Villegas2013} for ImageCLEF and \cite{Duygulu2002} for COREL5k; see also extra details below). The relative approximation error (RE) of a given DMN w.r.t. its underlying  DKN is measured (on a given set ${ \cal T}  \subset {\cal X}$) as
\begin{equation}
\small
 \textrm{RE} = \frac{1}{|{\cal T}|^2} \sum_{\x,\x' \in {\cal T}}  \frac{|{\langle \hat{\phi}_1^{(3)}(\x),\hat{\phi}_1^{(3)}(\x')\rangle}-\kappa_1^{(3)}(\x, \x')|}{|\langle \hat{\phi}_1^{(3)}(\x),\hat{\phi}_1^{(3)}(\x')\rangle|+|\kappa^{(3)}_1(\x, \x')|} \times 100\%,
\end{equation}
\noindent In the remainder of this section,  we show different evaluation measures (discrimination power,  RE and efficiency) on ImageCLEF and COREL5k benchmarks;  note that efficiency was measured on a Mac OS with Intel Core i5 processors.

\subsection{ImageCLEF benchmark} \label{sec:imageclef}
The ImageCLEF Photo Annotation benchmark~\cite{Villegas2013} includes more than 250k (training, dev and test) images belonging to 95 different concepts. As ground truth is available (released) only on the dev set (with 1,000 images), we learn DKNs and SVMs~\cite{Jiu2015} using only the dev set; the latter is split into two subsets: the first one used for DKN+SVM training while the other one for SVM testing. Given a concept and a test image, the decision about whether that concept is present in that test image depends on the score of a classifier; the latter corresponds to a ``one-versus-all'' SVM that returns a positive score if the concept is present in the test image and a negative score otherwise. The discrimination power of DKN and DMN (when combined with SVMs) is evaluated using the F-measure (defined as  harmonic means of recalls and precisions) both at the concept and the image levels (resp. denoted  MF-C and MF-S) as well as the Mean Average Precision (MAP) \cite{Villegas2013}; high values of these measures imply better performances. 

\indent In order to feed the inputs of DKN, we consider a combination of 10 visual features (provided by the ImageCLEF challenge organizers) and 4 elementary kernels (i.e. linear, polynomial with 2 orders, Gaussian\footnote{with a scale hyper-parameter set to be average Euclidean distance between data samples and their neighbors.} and histogram intersection) and we train a three-layer DKN with 40 input and 80 hidden units in a supervised way following the scheme in~\cite{Jiu2015}; the only difference w.r.t. \cite{Jiu2015} resides in the hyperbolic tangent activation function which is used to provide a better numerical stability and convergence when training  DKN. 

\noindent {\bf Initial DMNs.} Assuming the weights $\{\w_{p,q}^{(l-1)}\}$ of three-layer DKN known, we build its equivalent DMN (referred to as initial DMN) as shown in Section \ref{sec:dmn}.  In these experiments, we consider two random samplings of the subset  $\cal S$ -- from the dev set with $|{\cal S}|=500$ and $|{\cal S}|=1000$ -- in order to build the initial DMN (see Section~\ref{sec:dmn} and Eqs.~(\ref{equ:fullfeature}), (\ref{equa:projection})). According to Table~\ref{dmnperformance}, we observe that the performance of the initial DMN -- with $|{\cal S}|=500$ -- slightly degrades compared to its underlying DKN;  indeed, MF-S and MF-C decrease by \textcolor{black}{1.3} and \textcolor{black}{2.6} pts respectively while MAP decreases by \textcolor{black}{6.0} pts. With $|{\cal S}|=1000$ performances of the initial DMN is clearly improved compared to the one with $|{\cal S}|=500$; we obtain a slight gain in MF-S and comparable performance in MF-C. \textcolor{black}{We also provide  a comparison of the discrimination power of initial DMN against shallow DKN (i.e two-layer DKN) using a supervised setting; Table~\ref{dmnperformance} clearly shows the superiority of initial DMN (when $|{\cal S}|=1000$).} The relative approximation error (RE) of the two initial DMNs (i.e., with $|{\cal S}|=500$ and $|{\cal S}|=1000$) are also shown in Table~\ref{tab:imageclefRE}; we evaluate these REs on ${\cal T}$ with a cardinality  ranging  from 2,000 to 10,000 samples. From these results, we observe that REs are comparably low on small sets; indeed, with $|{\cal T}|=2,000$, the obtained REs are equal to 0.94\% when $|{\cal S}|=500$ and 0.95\% when $|{\cal S}|=1000$. Higher REs are obtained on larger $\cal T$ and this clearly motivates the importance of fine-tuning in order to make REs (and thereby performances) of the learned DMN stable (and close to the underlying DKN).

\begin{table}[tbp]
	\centering
	\begin{tabular}{c|ccc}
	\hline
	Framework & MF-S & MF-C & MAP \\
	\hline
	\textcolor{black}{2-layer DKN} & 44.96  & 25.77  & 53.95 \\
	3-layer DKN & 46.23 & 30.00 & 55.73 \\
	\hline
	Initial DMN  \  ($|{\cal S}|=500$) & 44.92 & 27.39 & 49.75 \\
	Fine-tuned DMN ($|{\cal S}'|=2000$) & \textbf{45.05} & \textbf{27.51} & \textbf{49.80} \\
	Fine-tuned DMN ($|{\cal S}'|=3000$) & 44.94 & 27.40 & 49.80 \\
	Fine-tuned DMN ($|{\cal S}'|=4000$) & 45.06 & 27.44 & 49.79 \\	
	\hline
	 Initial DMN ($|{\cal S}|=1000$) & 47.73 & 29.40 & 53.15 \\

        Fine-tuned DMN ($|{\cal S}'|=2000$) & 47.79 & 29.68 & 52.89 \\
        Fine-tuned DMN ($|{\cal S}'|=3000$) & \textbf{47.95} & \textbf{29.80} & \textbf{53.32} \\
        Fine-tuned DMN ($|{\cal S}'|=4000$) & 47.70 & 29.30 & 53.33  \\      
	\hline
	\end{tabular}
	\vspace{0.1cm}
    \caption{\textcolor{black}{The \textcolor{black}{discrimination} power (in \%) of different DMNs w.r.t the underlying DKN; in these experiments, two initial DMNs are designed using 500 and 1000 samples.} \label{dmnperformance}}
\end{table}

\begin{table*}[ht]
	\centering
	\begin{tabular}{c|c|ccccccccc}
	\hline
	Configuration & $|{\cal S}'|$ & 2K & 3K & 4K & 5K & 6K & 7K & 8K & 9K & 10K \\
	\hline
	Initial DMN ($|{\cal S}|=500$) & - & 0.94 & 1.25 & 1.41 & 1.51 & 1.58 & 1.62 & 1.66 & 1.69 & 1.71 \\
        \multirow{5}{*}{Fine-tuned DMN}  
                               & 500 & 0.89 & 1.19 & 1.35 & 1.45 & 1.52 & 1.57 & 1.60 & 1.63 & 1.65 \\ 
                               & 1000 & 0.89& 1.20 & 1.36 & 1.46 & 1.53 & 1.58 & 1.61 & 1.64 & 1.66 \\
                               & 2000 & 0.42 & 0.46 & 0.50 & 0.52 & 0.54 & 0.56 & 0.57 & 0.58 & 0.59 \\	
                               & 3000 & 0.52 & 0.47 & 0.47 & 0.47 & 0.47 & 0.47 & 0.48 & 0.48 & 0.48 \\
                               & 4000 & 0.60 & 0.51 & 0.49 & 0.47 & 0.47 & 0.46 & 0.46 & 0.46 & 0.46 \\ 
        \hline
	\hline
        Initial DMN ($|{\cal S}|=1000$)& - & 0.95 & 1.27 & 1.44 & 1.54 & 1.62 & 1.67 & 1.70 & 1.74 & 1.76 \\ 
        \multirow{4}{*}{Fine-tuned DMN} & 1000 & 0.89 & 1.21 & 1.38 & 1.48 & 1.55 & 1.60 & 1.64 & 1.67 & 1.69 \\
                               & 2000 & 0.37 & 0.41 & 0.44 & 0.46 & 0.48 & 0.49 & 0.50 & 0.51 & 0.52 \\
                               & 3000 & 0.46 & 0.43 & 0.43 & 0.43 & 0.44 & 0.44 & 0.44 & 0.45 & 0.45 \\
                               & 4000 & 0.54 & 0.48 & 0.46 & 0.45 & 0.45 & 0.44 & 0.44 & 0.44 & 0.44 \\
        \hline
	\end{tabular}
       \vspace{0.1cm}
	\caption{Relative errors of initial and fine-tuned DMNs w.r.t. the DKN for different dataset cardinalities $|{\cal T}|$ (ranging from 2K to 10K) and when two different initializations are employed. \label{tab:imageclefRE}}
\end{table*}
\begin{figure}[thbp]
\begin{center}
\includegraphics[width=0.7\linewidth]{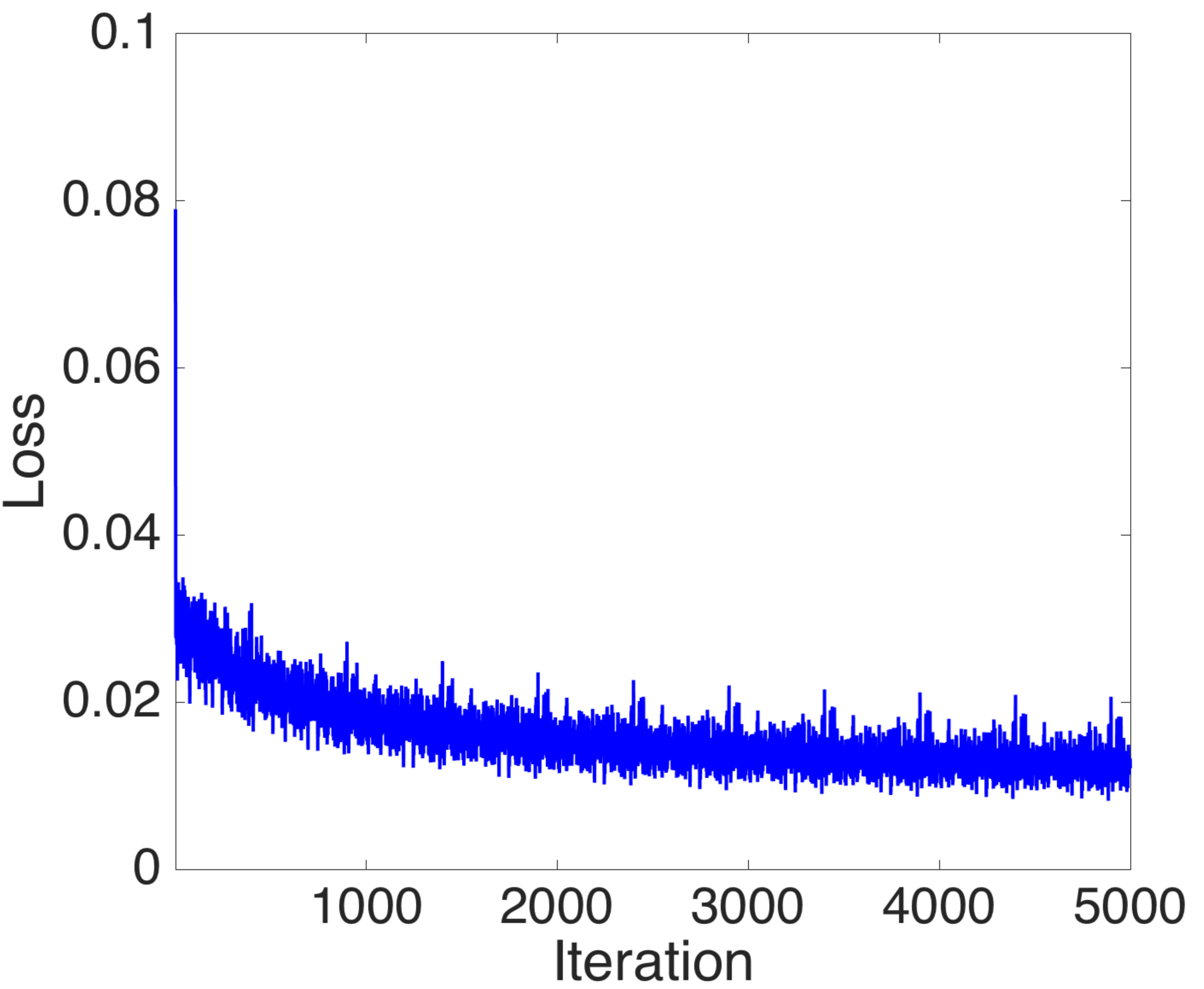}
\end{center}
\caption{\textcolor{black}{This figure shows the loss criterion in Eq.~(\ref{equa:lossfunction}) as the learning iterates when $|{\cal S}|=500$ and $|{\cal S}'|=2000$. }\label{fig:timeresults}} 
\end{figure}

\begin{figure}[thbp]
\begin{center}
\includegraphics[width=0.7\linewidth]{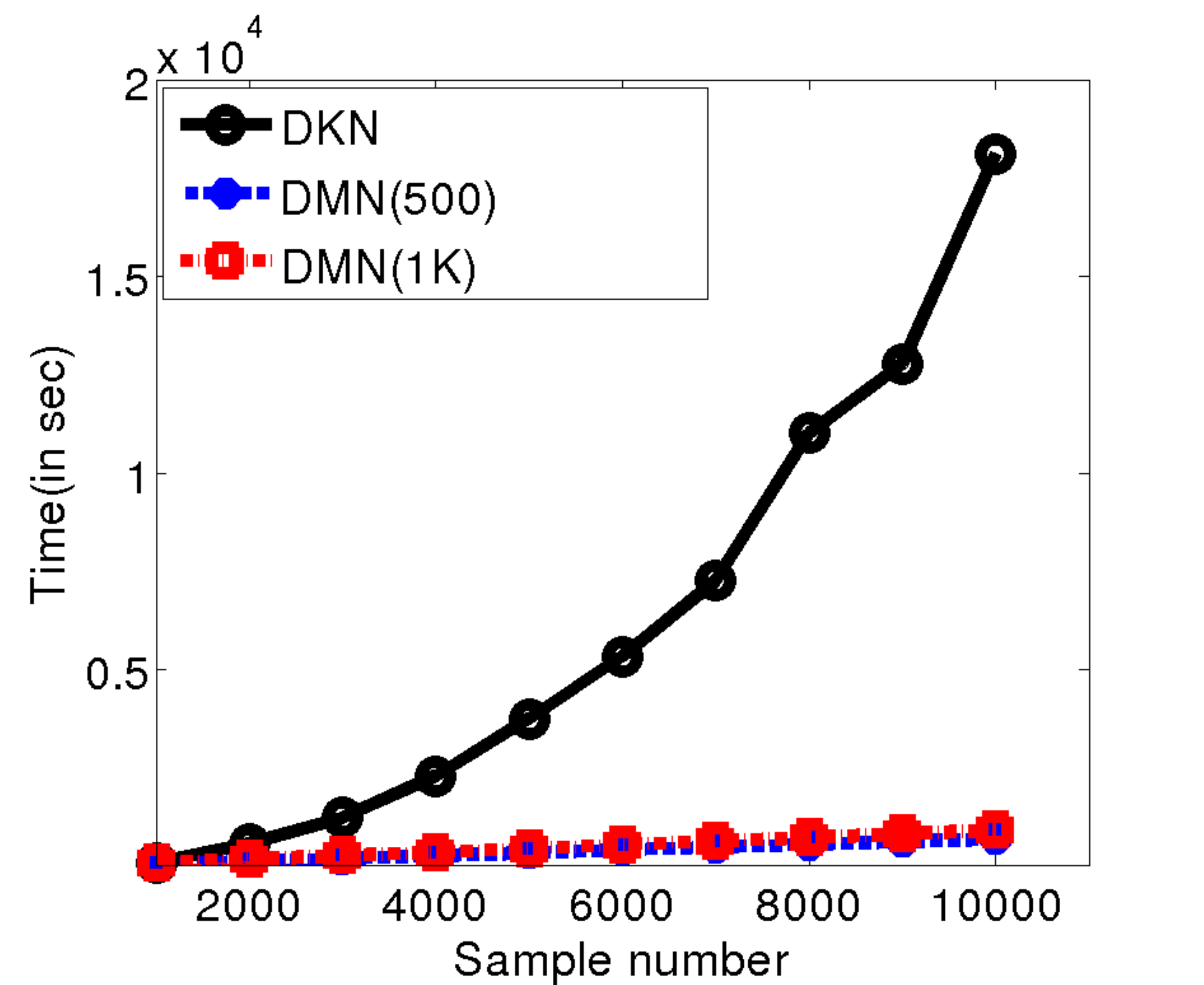}
\end{center}
\caption{\textcolor{black}{This figure shows a comparison of processing time \textcolor{black}{between} two different DMNs and their underlying DKN as $|{\cal T}|$ increases (with $|{\cal S}|=500$ and $|{\cal S}|=1000$) on ImageCLEF dataset.} \label{fig:timeresults-imageclef}} 
\end{figure}

\begin{table}[tbp]
	\centering
	\begin{tabular}{c|ccc}
	\hline
	  & $|{\cal T}|$& 50K & 100K  \\
	\hline
	\textcolor{black}{3-layer} DKN & Time & 40.4 hrs & 160.3 hrs  \\
	\hline
	Fine-tuned DMN & Time & 1.1 hrs &  2.4 hrs  \\
	$|{\cal S}|=500$\-, $|{\cal S}'|=4000$ & RE & 0.46\% & 0.46\%\\
	\hline
      Fine-tuned DMN  & Time & 1.3 hrs  & 2.8 hrs \\        
       $|{\cal S}|=1000$\-, $|{\cal S}'|=4000$ & RE & 0.45\% & 0.45\% \\
	\hline
	\end{tabular}
    \vspace{0.1cm}
    \caption{This table shows a comparison of processing time and relative errors between the DKN and the fine-tuned DMN on $50K$ and $100K$ images of ImageCLEF.``hrs'' stands for ``hours''.\label{dkndmnonlargedata}}
\end{table}
 
\noindent \noindent {\bf Fine-tuned DMNs.} In order to fine-tune the parameters of DMN, we use the learning procedure presented in Section~\ref{sec:unsupervisedlearning}. We consider an unlabeled set ${\cal S}'$ (\textcolor{black}{with $|{\cal S}'|$ ranging from 1,000 to 4,000}) and we sample 100,000 pairs from ${\cal S}' \times  {\cal S}'$  to minimize criterion (\ref{equa:lossfunction}) using gradient descent with a step-size empirically set to $10^{-6}$, a mini-batch size equal to 200 and a max number of iterations set to 5,000 (see Fig.~\ref{fig:timeresults}). \\
\indent As shown in Table~\ref{dmnperformance}, we observe that the discrimination power of different DMNs remains stable (with a slight gain in MF-S when $|{\cal S'}|=3000$) w.r.t. their underlying DKNs, and this naturally follows the noticeably small REs of  the fine-tuned DMNs (see Table~\ref{tab:imageclefRE}). The latter are further positively impacted when $|{\cal S}'|$ becomes larger; for instance, when increasing $|{\cal S}'|$  from 1,000 to 4,000, the RE decreases significantly (particularly when $|{\cal T}|=10,000$). Moreover, and in contrast to the initial DMNs, the fine-tuned DMNs are less sensitive to $|{\cal T}|$ as shown through the observed REs which remain stable w.r.t. $|{\cal T}|$.
 
\indent Finally, we measure the gain in efficiency obtained with DMNs against DKNs. From \textcolor{black}{Fig.~\ref{fig:timeresults-imageclef}}, we observe that DMN is (at least) an order of magnitude faster compared to its DKN; for instance, with 10,000 samples, DKN requires more than 15,000 seconds in order to compute kernel values while  DMN requires  less than 1,000 seconds. Table~\ref{dkndmnonlargedata} also provides a comparison of efficiency and RE on much larger sets (resp.~50K and 100K) randomly sampled from the (unlabeled) training set of ImageCLEF; a significant improvement in efficiency is observed. In other words, the complexity of evaluating DMNs is linear while for DKN it is quadratic. These results clearly corroborate the fact that the proposed DMNs are as effective as DKNs while being highly efficient especially on large scale datasets. 

\subsection{COREL5k benchmark} \label{sec:corel5k}
The COREL5k database introduced in~\cite{Duygulu2002} is another benchmark which is widely used for image annotation. In this database, 4,999 images are collected and a vocabulary of 200 keywords is used for annotation.  This set is split into two parts;  the first one includes 4,500 images for training and the second one 499 images for testing. As for ImageCLEF, the task is again to assign a list of keywords for each image in the test set.\\
\indent Each image in COREL5k is described using 15 types of INRIA features \cite{Guillaumin2009} including:  GIST features, 6 color histograms for RGB, HSV, LAB in two spatial layouts, 8 bag-of-features based on SIFT and robust hue descriptors in two spatial layouts. Following the standard  protocol defined on COREL5k \cite{Duygulu2002}, each test image is annotated with up to 5 keywords and performances (discrimination power of image classification/annotation) are measured by the mean precision  and recall over keywords (referred to as $\mathbf{P}$ and  $\mathbf{R}$ respectively) as well as the number of keywords with non-zero recall value (denoted $\mathbf{N_{+}}$); again, higher values of these measures imply better performances. 
 
\begin{table}[htbp]
	\centering
	\begin{tabular}{c|ccc}
	\hline
	Framework & $\mathbf{R}$ & $\mathbf{P}$ & $\mathbf{N_{+}}$ \\
	\hline
	3-layer DKN & 37.65 & 25.49 & 158 \\
	\hline
	Initial DMN ($|{\cal S}|=500$) & 31.30 & 18.67 & 155 \\
	Fine-tuned DMN ($|{\cal S}'|=2000$) & 31.34 & 18.54 & 155 \\
	Fine-tuned DMN ($|{\cal S}'|=3000$) & 31.62 & 18.43 & 153 \\
	Fine-tuned DMN ($|{\cal S}'|=4000$) & 31.18 & 19.04 & 155 \\	
	Fine-tuned DMN ($|{\cal S}'|=4999$) & \textbf{31.65} & \textbf{19.13} & \textbf{157} \\
        \hline

	Initial DMN ($|{\cal S}|=700$) & 32.31 & 19.39 & 155 \\
	Fine-tuned DMN ($|{\cal S}'|=2000$) & 32.57 & 19.82 & 157 \\
	Fine-tuned DMN ($|{\cal S}'|=3000$) & 33.05 & \textbf{20.88} & \textbf{159} \\
	Fine-tuned DMN ($|{\cal S}'|=4000$) & 33.08 & 20.40 & 158 \\	
	Fine-tuned DMN ($|{\cal S}'|=4999$) & \textbf{33.30} & 20.18 & 158 \\
        \hline        
	\end{tabular}
	\vspace{0.1cm}
       \caption{The \textcolor{black}{discrimination} power of different DMNs w.r.t the underlying DKN on COREL5k; in these experiments, two initial DMNs are designed using 500 and 700 samples. \label{coreldmnperformance}}
\end{table}

\begin{table}[ht]
	\centering
	\begin{tabular}{c|c|ccccc}
	\hline
	 Framework & $|{\cal S}'|$ & 2K & 3K & 4K & 4999 \\
	\hline
	Initial DMN $|{\cal S}|=500$ & - & 2.45 & 2.41 & 2.35 & 2.26 \\
        \multirow{6}{*}{Fine-tuned DMN} & 500 & 1.22 & 1.28 & 1.32 & 1.37 \\
                               & 1000 & 1.23 & 1.35 & 1.40 & 1.42 \\
                               & 2000 & 1.12 & 1.15 & 1.19 & 1.19 \\	
                               & 3000 & 1.14 & 1.12 & 1.13 & 1.12 \\
                               & 4000 & 1.18 & 1.14 & 1.11 & 1.10 \\ 
                               & 4999 & 1.18 & 1.14 & 1.12 & 1.10 \\                             
        \hline
        Initial DMN $|{\cal S}|=700$ & - & 2.43 & 2.39 & 2.33 & 2.24 \\
         \multirow{6}{*}{Fine-tuned DMN} & 700 & 1.30 & 1.42 & 1.48 & 1.51 \\
                               & 1000 & 1.22 & 1.35 & 1.42 & 1.44 \\
                               & 2000 & 1.09 & 1.13 & 1.17 & 1.18 \\	
                               & 3000 & 1.11 & 1.10 & 1.11 & 1.11 \\
                               & 4000 & 1.16 & 1.12 & 1.09 & 1.08 \\ 
                               & 4999 & 1.16 & 1.12 & 1.10 & 1.08 \\                             
        \hline 
	\end{tabular}
	\vspace{0.1cm}
	\caption{Relative errors of initial and fine-tuned DMNs (w.r.t. the underlying DKN) on COREL5k as  $|{\cal T}|$ increases (with values ranging from 2K to 4999) \label{tab:coreltimeresults}}
\end{table}

\begin{figure}[ht]
\centering
\begin{center}
\includegraphics[width=0.7\linewidth]{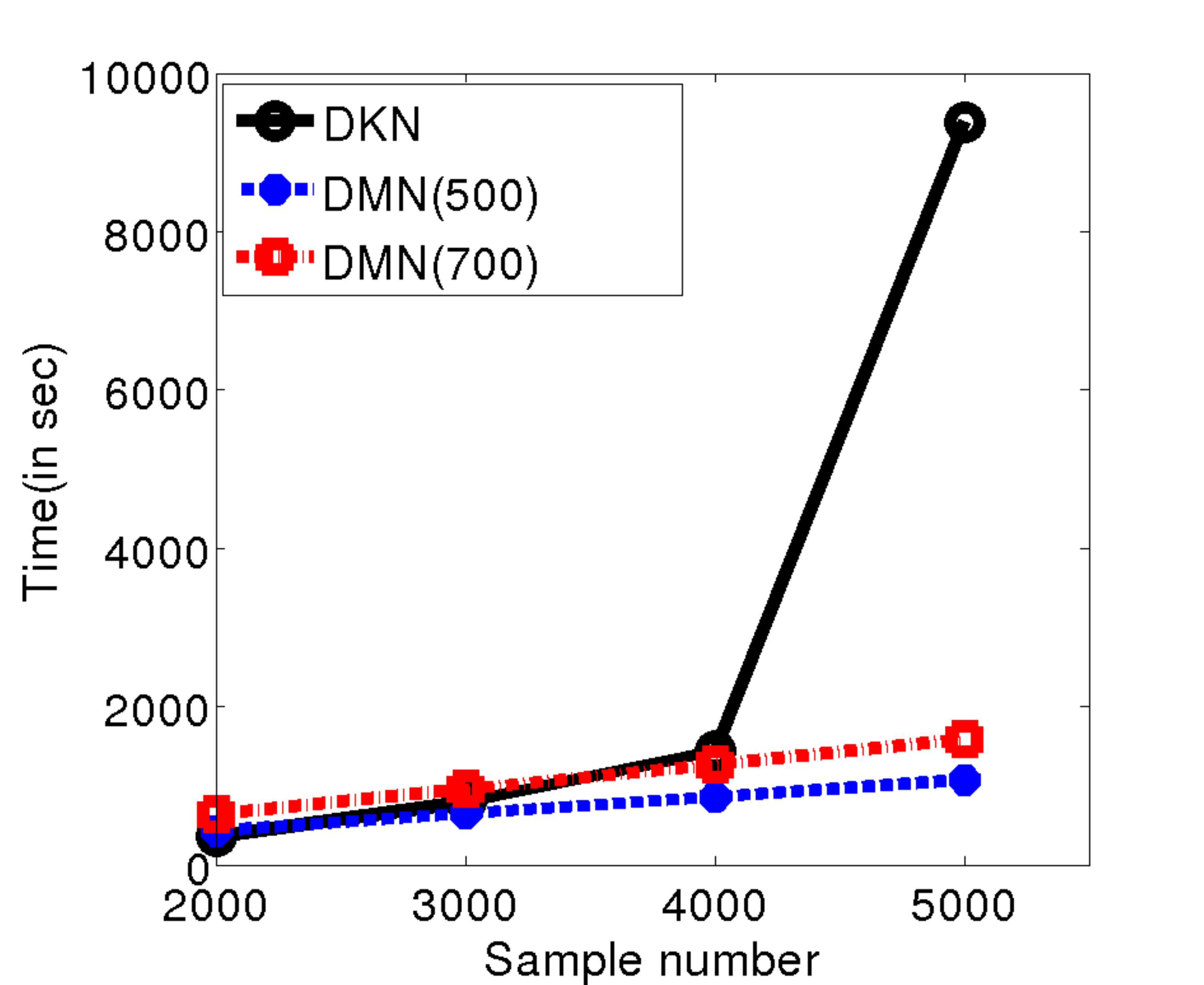}
\end{center}
\caption{\textcolor{black}{Comparison of processing time between two approximated DMNs (with $|{\cal S}|=500$ and $|{\cal S}|=700$) and their underlying DKN as $|{\cal T}|$ increases on COREL5k dataset.}\label{fig:coreltimecomparison}} 
\end{figure}

\begin{table}[tbp]
    \centering
\resizebox{0.5\textwidth}{!}{
    \begin{tabular}{c|c|c|ccc}
    \hline
          Method & Learned & context & $\mathbf{R}$ & $\mathbf{P}$ & $\mathbf{N_{+}}$ \\
          & \textcolor{black}{Input.} feat. & &  & &  \\
      \hline
          wTKML~\cite{Vo2012}&  no & yes & 42 & 21 & 173 \\
          LDMKL~\cite{ZhangDPR2012}& no & yes & 44 & 29 & 179 \\
          CNN-R~\cite{Murthyicmr2015} &yes & yes &  41.3 & 32.0 & 166 \\
          \hline
          \textcolor{black}{3-layer} DKN+SVM~\cite{Jiutip2017}& no & no & 37.7 & 25.5 & 158 \\
          Init. DMN+SVM ($|{\cal S}|=700$)& no & no & 32.3 & 19.3 & 155 \\
          FT DMN+SVM ($|{\cal S}|=700$)& no & no & 33.1 & 20.9 & 159 \\
          \textcolor{black}{Init. DMN+SVM ($|{\cal S}|=1200$)} & no & no & 34.0 & 20.9 & 162 \\
          \textcolor{black}{FT DMN+SVM ($|{\cal S}|=1200$)} & no & no & \textcolor{black}{34.7} & \textcolor{black}{21.0} & \textcolor{black}{168} \\
      \hline
          ResNet\cite{HeZhangCVPR2016} + SVM & yes & no & 34.5 & 21.8 & 161 \\
          \hline
      \textcolor{black}{3-layer} DKN+SVM~\cite{Jiutip2017} & yes & no & 42.6 & 24.9 & 180 \\
       Init. DMN+SVM ($|{\cal S}|=700$) & yes & no &  36.1 & 21.7 & 166 \\
          FT DMN+SVM ($|{\cal S}|=700$) & yes & no & 36.8 & 22.4 & 165 \\
      Init. DMN+SVM ($|{\cal S}|=1000$) & yes & no &  37.4 & 21.6 & 162 \\
          FT DMN+SVM ($|{\cal S}|=1000$) & yes & no & 37.7 & 22.3 & 164 \\    
       \textcolor{black}{Init. DMN+SVM ($|{\cal S}|=1200$)} & yes & no & 37.8 & 23.2 & 167 \\
          \textcolor{black}{FT DMN+SVM ($|{\cal S}|=1200$)} & yes & no & \textcolor{black}{38.9} & \textcolor{black}{23.2} & \textcolor{black}{169} \\               
        \hline
    \end{tabular}
    }
    \vspace{0.1cm}
    \caption{\textcolor{black}{Extra comparison of the proposed DMN w.r.t different settings as well as the related work. In these experiments,  $|{\cal S}'|=3000$ and different $|{\cal S}|$ are used. In this table, FT stands for Fine-Tuned.}\label{corelcomparisonother}}
\end{table}
 
As in ImageCLEF (see section~\ref{sec:imageclef}), we use 4 elementary kernels for each feature: linear, order two polynomial, RBF (with a scale parameter set to the average distance between data) and histogram intersection; in total, we use 60 different elementary kernels as inputs to the 3-layer DKN. We also use the same DKN architecture on {COREL5k}  with a slight difference in the number of units in the hidden layers  (equal to 120 instead of 80 in ImageCLEF). Again, the weights of DKN are learned using the semi-supervised learning procedure presented in~\cite{Jiu2015} where the similarity between images is computed by the heat kernel (with a width set to the mean distance between neighbors).
An ensemble of ``one-versus-all'' SVM classifiers is trained on top of DKN for each category. The average decision score from all the classifiers is taken as a final score for a given category. In order to avoid the severe imbalanced class distributions in SVM training, we adopt a sampling strategy that  randomly selects a subset of negative samples whose cardinality is equal to the number of positive training samples. Hence, each classifier is learned using all the positive data and a random subset of negative data. The discrimination power of the learned DKNs+SVMs is shown in Table~\ref{coreldmnperformance}. \\

\noindent {\bf Initial and fine-tuned DMNs.}  Assuming the weights $\{\w_{p,q}^{(l-1)}\}$ of DKN known (learned), we build the initial DMN as shown in Section~\ref{sec:dmn}.  We consider two random samplings of the subset $\cal S$ -- from the training set with $|{\cal S}|=500$ and $|{\cal S}|=700$  -- in order to build the initial DMN.  We also use the learning procedure presented in Section~\ref{sec:unsupervisedlearning} in order to fine-tune the parameter of the DMN. We consider an unlabeled set ${\cal S}'$ which includes up to 4,999 samples (i.e. the whole COREL5K set); again we sample 100,000 pairs  in order to minimize the criterion in Eq.~(\ref{equa:lossfunction}) using gradient descent with a step-size empirically set to $10^{-6}$, a mini-batch size equal to  200 and a max number of iterations set to 5,000. 
  
According to Table~\ref{coreldmnperformance}, we observe that the performances of the initial DMNs ($\mathbf{R}$, $\mathbf{P}$ and $\mathbf{N_{+}}$) again degrade compared to  their underlying DKNs as a result of the high RE of these DMNs. This degradation in performances is also amplified by the scarceness of training data for SVM learning in COREL5k (in contrast to ImageCLEF) especially when the RE is relatively large (see Table.~\ref{tab:coreltimeresults}). However, the discrimination power is improved when more data are used to design these DMNs (i.e., with $|{\cal S}|=700$ \textcolor{black}{and also $|{\cal S}|=1200$ in Table~\ref{corelcomparisonother}}). Furthermore, fine-tuning DMNs reduces the RE as $|{\cal S}'|$ increases, and makes RE stable even with a relatively large $|{\cal T}|$, so RE (on COREL5k) behaves similarly compared to ImageCLEF. 
Finally, Fig.~\ref{fig:coreltimecomparison} shows a comparison of processing time between DMN and DKN. It is easy to see that when $|{\cal T}|$ is small, the processing times of DKN and DMN are comparable. However, when $|{\cal T}|$  reaches large values (e.g., $|{\cal T}|=4999$), DMN becomes an order of magnitude faster than its underlying DKN while maintaining a comparable accuracy. \\

\noindent {\bf Extra comparisons.} We further compare the performance of DMNs against closely related kernel-based methods (namely wTKML~\cite{Vo2012} and LDMKL~\cite{ZhangDPR2012}) as well as convolutional neural networks (mainly CNN-R~\cite{Murthyicmr2015}). wTKML~\cite{Vo2012} learns explicit and transductive kernel maps using a priori knowledge taken from the semantic and geometric (statistical) dependencies between classes while LDMKL~\cite{ZhangDPR2012} combines Laplacian SVM with deep kernel networks using an ``end-to-end'' framework. CNN-R~\cite{Murthyicmr2015} combines deep features from Caffe-Net with \textcolor{black}{word embedding features from Word2Vec}; as introduced in the literature, these related methods leverage different sources of contexts and a priori knowledge while our method does not. \\
In our experiments (see Table~\ref{corelcomparisonother}), we use four elementary kernels (linear, polynomial, RBF and HI) combined with different features as inputs to the designed  DKN and DMN networks: ``handcrafted features'' including GIST and SIFT  and ``learned features'' taken from ResNet~\cite{HeZhangCVPR2016}  (pretrained on the ImageNet) which is a very deep architecture consisting of 152 layers; the 2048 dimensional features of the last pooling layer are used in our annotation task. Using all these elementary kernels and features, we first train a DKN in a supervised way according to \cite{Jiutip2017}, then we design and fine-tune its associated DMNs with $|{\cal S}|=700$ and  $|{\cal S}'|=3000$ (as done in Table.~\ref{coreldmnperformance}). \\
\indent From the results shown in Table~\ref{corelcomparisonother}, first, we observe that the use of ResNet features as inputs to our DMN framework provides a clear gain compared to the use of handcrafted features. Second, fine-tuning DMNs brings a clear gain compared to the initial DMNs as well as ResNet. Our DKN (and its DMN variant) can even catch (and sometimes outperform) the aforementioned related work which again relies on different contextual clues, in contrast to our method. We believe that considering context will further enhance the performance of DKNs and their associated DMNs, but this is out of the main scope of this paper and will be investigated as a future work. \\ 
\indent Finally, Fig.~\ref{fig:imageinstance} shows examples of annotation results, on the test set, obtained using the learned DMNs and the underlying DKNs on ImageCLEF and COREL5k datasets. From these figures, DMNs behave similarly, \textcolor{black}{w.r.t.} DKNs, with an extra advantage of being computationally more efficient especially on COREL5k (as shown in \textcolor{black}{Table}~\ref{tab:testtimecomparison}); whereas the computational complexity of DKN evaluation scales linearly \textcolor{black}{w.r.t.} the number of support vectors (which is an order of magnitude larger on COREL5k \textcolor{black}{w.r.t.} ImageCLEF: \textcolor{black}{4,500 versus 500}), the computational complexity of DMN evaluation grows slowly and remains globally stable \textcolor{black}{w.r.t.} the number of support vectors (which is again an order of magnitude larger on COREL5k). These results are also consistent with those already shown in Fig.~\ref{fig:timeresults-imageclef} \textcolor{black}{and Fig.~\ref{fig:coreltimecomparison}}.

\begin{table}[htbp]
    \centering
    \begin{tabular}{c|c|c}
    \hline
    Dataset & Framework & time (in sec) \\
    \hline
    \multirow{3}{*}{ImageCLEF} & DKN & 0.68  \\
    & Fine-tuned DMN ($|{\cal S}|=500$) & 0.57 \\
    & Fine-tuned DMN ($|{\cal S}|=1000$)& 0.95\\
    \hline
    \multirow{5}{*}{COREL5k} & DKN & 10.39  \\
    & Fine-tuned DMN ($|{\cal S}|=500$) & 1.22 \\
    & Fine-tuned DMN ($|{\cal S}|=700$) & 1.58 \\
    & Fine-tuned DMN ($|{\cal S}|=1000$) & 2.51 \\
    & Fine-tuned DMN ($|{\cal S}|=1200$) & 3.67 \\
    \hline
    \end{tabular}
    \vspace{0.1cm}
       \caption{\textcolor{black}{Comparison of the average processing time per test image (excluding feature extraction) on ImageCLEF and COREL5k datasets.} \label{tab:testtimecomparison}}
\end{table}

\begin{figure*}[ht]
\centering
\begin{center}
\includegraphics[width=1\linewidth]{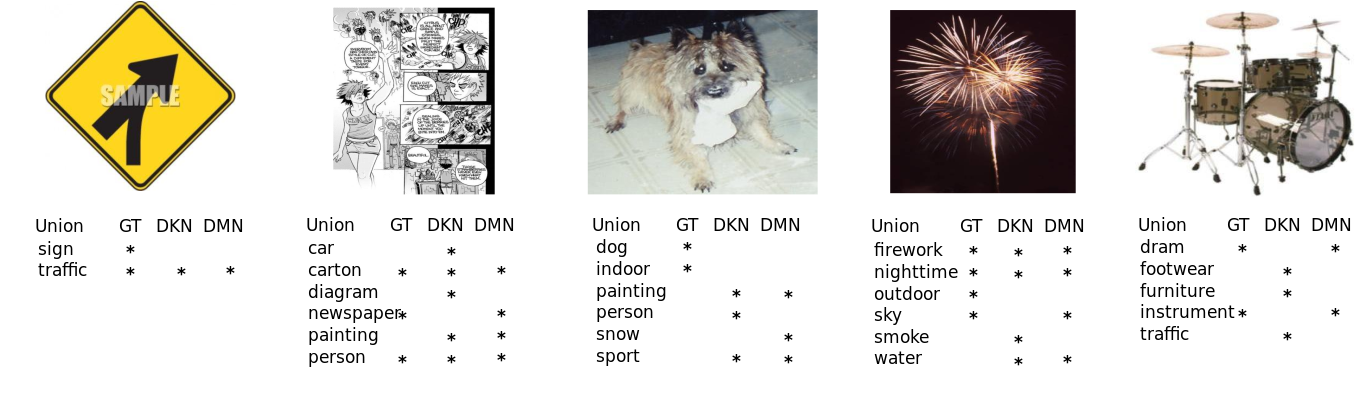}
\hspace{0.1cm}
\includegraphics[width=1\linewidth]{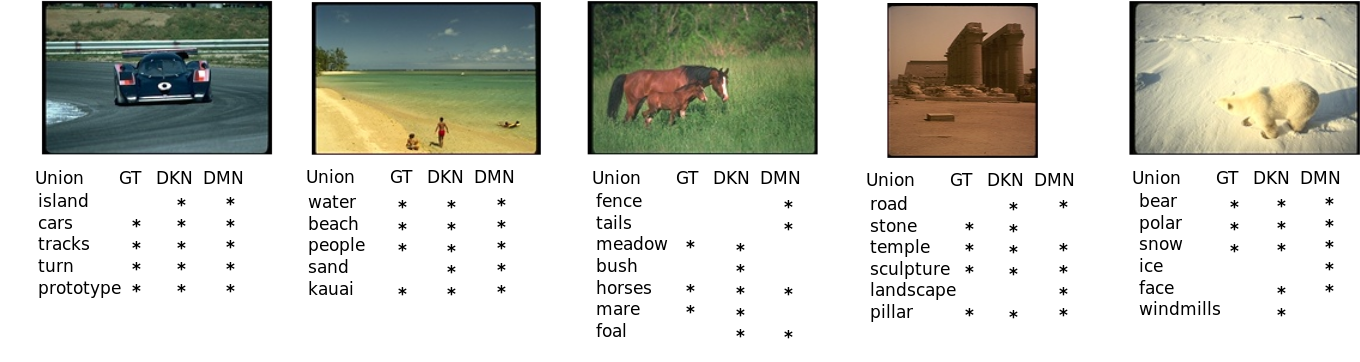}
\end{center}
\caption{{Examples of annotation results  using DKNs and their ''Fine-tuned'' DMN variants on ImageCLEF (top) and COREL5k (bottom).``GT'' stands for ground-truth keywords and the symbol ``*'' stands for the presence of a keyword in a given test image.} \label{fig:imageinstance}} 
\end{figure*}

\section{Conclusion}
In this paper we introduced a novel method that transforms deep kernel networks into highly efficient deep map networks. The proposed method is greedy and proceeds layer-wise by expressing p.s.d kernels in different (input, intermediate, and output) layers of DKN as inner products involving explicit maps. These explicit maps are either exactly designed for some input kernels (including linear and polynomial) or tightly approximated for others (including intermediate and output kernels in DKN). We also introduced an unsupervised fine-tuning algorithm that benefits from large unlabeled sets in order to further enhance the generalization capacity of DMNs. Extensive experiments in image annotation, using the challenging ImageCLEF and COREL5k benchmarks, clearly demonstrate the effectiveness of DMNs and their high efficiency. 

\section*{Acknowledgment}

This work was supported in part by a grant from the research agency ANR (Agence Nationale de la Recherche) under the MLVIS project, ANR-11-BS02-0017.
\ifCLASSOPTIONcaptionsoff
  \newpage
\fi

\bibliographystyle{IEEEtran}
\bibliography{references}

\begin{thebibliography}{10}
\providecommand{\url}[1]{#1}
\csname url@samestyle\endcsname
\providecommand{\newblock}{\relax}
\providecommand{\bibinfo}[2]{#2}
\providecommand{\BIBentrySTDinterwordspacing}{\spaceskip=0pt\relax}
\providecommand{\BIBentryALTinterwordstretchfactor}{4}
\providecommand{\BIBentryALTinterwordspacing}{\spaceskip=\fontdimen2\font plus
\BIBentryALTinterwordstretchfactor\fontdimen3\font minus
  \fontdimen4\font\relax}
\providecommand{\BIBforeignlanguage}[2]{{%
\expandafter\ifx\csname l@#1\endcsname\relax
\typeout{** WARNING: IEEEtran.bst: No hyphenation pattern has been}%
\typeout{** loaded for the language `#1'. Using the pattern for}%
\typeout{** the default language instead.}%
\else
\language=\csname l@#1\endcsname
\fi
#2}}
\providecommand{\BIBdecl}{\relax}
\BIBdecl

\bibitem{Caputo2004}
B.~Caputo, C.~Wallraven, and M.-E. Nilsback, ``Object categorization via local
  kernels,'' in \emph{ICPR}, 2004.

\bibitem{Lyu2005}
S.~Lyu, ``Mercer kernels for object recognition with local features,'' in
  \emph{CVPR}, 2005.

\bibitem{Grauman2007}
K.~Grauman and T.~Darrell, ``The pyramid match kernel: Efficient learning with
  sets of features,'' \emph{JMLR}, vol.~8, pp. 725--760, 2007.

\bibitem{Qi2007}
X.~Qi and Y.~Han, ``Incorporating multiple svms for automatic image
  annotation,'' \emph{IEEE Transcations on Knowledge and Data Engineering},
  vol.~40, 2007.

\bibitem{sahbi2002coarse}
H.~Sahbi and N.~Boujemaa, ``Coarse-to-fine support vector classifiers for face
  detection,'' in \emph{Pattern Recognition, 2002. Proceedings. 16th
  International Conference on}, vol.~3.\hskip 1em plus 0.5em minus 0.4em\relax
  IEEE, 2002, pp. 359--362.

\bibitem{sahbi2003coarse}
H.~Sahbi, ``Coarse-to-fine support vector machines for hierarchical face
  detection,'' Ph.D. dissertation, PhD thesis, Versailles University, 2003.

\bibitem{Weinberger2004}
K.~Q. Weinberger, F.~Sha, and L.~K. Saul, ``Learning a kernel matrix for
  nonlinear dimensionality reduction,'' in \emph{ICML}, 2014.

\bibitem{Vapnik1998}
V.~Vapnik, ``Statistical learning theory,'' \emph{Wiley, New York}, 1998.

\bibitem{ShaweTaylor2004}
J.~Shawe-Taylor and N.~Cristianini, ``Kernel methods for pattern analysis,''
  \emph{Cambriage University Press}, 2004.

\bibitem{sahbi2008context}
H.~Sahbi, J.-Y. Audibert, J.~Rabarisoa, and R.~Keriven, ``Context-dependent
  kernel design for object matching and recognition,'' in \emph{Computer Vision
  and Pattern Recognition, 2008. CVPR 2008. IEEE Conference on}.\hskip 1em plus
  0.5em minus 0.4em\relax IEEE, 2008, pp. 1--8.

\bibitem{wang2013directed}
L.~Wang and H.~Sahbi, ``Directed acyclic graph kernels for action
  recognition,'' in \emph{Computer Vision (ICCV), 2013 IEEE International
  Conference on}.\hskip 1em plus 0.5em minus 0.4em\relax IEEE, 2013, pp.
  3168--3175.

\bibitem{Lanckriet2004}
G.~Lanckriet, N.~Cristianini, P.~Bartlett, L.~E. Ghaoui, and M.~I. Jordan,
  ``Learning the kernel matrix with semi-definite programming,'' \emph{JRML},
  vol.~5, pp. 27--72, 2004.

\bibitem{Yu2009nips}
K.~Yu, W.~Xu, and Y.~Gong, ``Deep learning with kernel regularization for
  visual recognition,'' in \emph{NIPS}, 2009, pp. 1889--1896.

\bibitem{Corinna2010}
C.~Corinna, M.~Mehryar, and R.~Afshin, ``Two-stage learning kernel
  algorithms,'' in \emph{ICML}, 2010.

\bibitem{Sahbi2011}
H.~Sahbi, J.-Y. Audibert, and R.~Keriven, ``Context-dependent kernels for
  object classification,'' \emph{PAMI}, vol.~33, pp. 699--708, 2011.

\bibitem{Sahbi2011a}
H.~Sahbi and X.~Li, ``Context-based support vector machines for interconnected
  image annotation,'' in \emph{ACCV}, 2011, pp. 214--227.

\bibitem{tollari2008comparative}
S.~Tollari, P.~Mulhem, M.~Ferecatu, H.~Glotin, M.~Detyniecki, P.~Gallinari,
  H.~Sahbi, and Z.-Q. Zhao, ``A comparative study of diversity methods for
  hybrid text and image retrieval approaches,'' in \emph{Workshop of the
  Cross-Language Evaluation Forum for European Languages}.\hskip 1em plus 0.5em
  minus 0.4em\relax Springer, 2008, pp. 585--592.

\bibitem{boujemaa2004visual}
N.~Boujemaa, F.~Fleuret, V.~Gouet, and H.~Sahbi, ``Visual content extraction
  for automatic semantic annotation of video news,'' in \emph{the proceedings
  of the SPIE Conference, San Jose, CA}, vol.~6, 2004.

\bibitem{Bach2004}
F.~Bach, G.~Lanckriet, and M.~Jordan, ``Multiple kernel learning, conic
  duality, and the smo algorithm,'' in \emph{ICML}, 2004.

\bibitem{Rakotomamonjy2008}
A.~Rakotomamonjy, F.~Bach, C.~S., and G.~Yves, ``Simplemkl,'' \emph{JMLR},
  vol.~9, pp. 2491--2521, 2008.

\bibitem{Sonnenburg2006}
S.~Sonnenburg, G.~R\"{a}tsch, C.~Schafer, and B.~Sch\"{o}lkopf, ``Large scale
  multiple kernel learning,'' \emph{JMLR}, vol.~7, pp. 1531--1565, 2006.

\bibitem{Bach2009}
F.~Bach, ``Exploring large feature spaces with hierarchical multiple kernel
  learning,'' in \emph{NIPS}, 2009, pp. 1--9.

\bibitem{Cortes2009}
C.~Cortes, M.~Mohri, and A.~Rostamizadeh, ``Learning non-linear combinations of
  kernels,'' in \emph{NIPS}, 2009, pp. 1--9.

\bibitem{LeCun98}
Y.~LeCun, L.~Botto, Y.~Bengio, and P.~Haffner, ``Gradient-based learning
  applied to document recognition,'' \emph{Proceedings of IEEE}, vol.~86,
  no.~11, pp. 2278--2324, 1998.

\bibitem{Krizhevsky2012}
A.~Krizhevsky, I.~Sutskever, and G.~E. Hinton, ``Imagenet classification with
  deep convolutional neural networks,'' in \emph{NIPS}, 2012.

\bibitem{Farabet2013}
C.~Farabet, C.~Couprie, L.~Najman, and Y.~LeCun, ``Learning hierarchical
  features for scene labeling,'' \emph{PAMI}, vol.~35, no.~8, pp. 1915--1929,
  2013.

\bibitem{Cho2009}
Y.~Cho and L.~Saul, ``Kernel methods for deep learning,'' in \emph{NIPS}, 2009,
  pp. 1--9.

\bibitem{Zhuang2011a}
J.~Zhuang, I.~Tsang, and S.~Hoi, ``Two-layer multiple kernel learning,'' in
  \emph{ICML}, 2011, pp. 909--917.

\bibitem{Jiu2015}
M.~Jiu and H.~Sahbi, ``Semi supervised deep kernel design for image
  annotation,'' in \emph{ICASSP}, 2015.

\bibitem{Williams2001}
C.~Williams and M.~Seeger, ``Using the nystr\"{o}m method to speed up kernel
  machines,'' in \emph{NIPS}, 2001.

\bibitem{Drineas2005}
P.~Drineas and M.~W. Mahoney, ``On the nystrom method for approximating a gram
  matrix for improved kernel-based learning,'' \emph{J. Mach. Learn. Res.},
  vol.~6, pp. 2153--2175, Dec. 2005.

\bibitem{Kumar2012}
S.~Kumar, M.~Mohri, and A.~Talwalkar, ``Sampling methods for the nystr\"{o}m
  method,'' \emph{J. Mach. Learn. Res.}, vol.~13, no.~1, pp. 981--1006, Apr.
  2012.

\bibitem{Rahimi2007}
A.~Rahimi and B.~Recht, ``Random features for large-scale kernel machines,'' in
  \emph{NIPS}, 2007.

\bibitem{LiIonescu2010}
F.~Li, C.~Ionescu, and C.~Sminchisescu, ``Random fourier approximations for
  skewed multiplicative histogram kernels,'' in \emph{DAGM conference Pattern
  Recognition}, 2010.

\bibitem{Vedaldi2012}
A.~Vedaldi and A.~Zisserman, ``Efficient additive kernels via explicit feature
  maps,'' \emph{IEEE Transactions on PAMI}, vol.~34, 2012.

\bibitem{HuangICASSP2013}
L.~Deng, M.~Hasegawa-Johnson, and X.~He, ``Random features for kernel deep
  convex network,'' in \emph{ICASSP}, 2013, pp. 3143--3147.

\bibitem{Mairal2014}
J.~Mairal, P.~Koniusz, Z.~Harchaoui, and C.~Schmid, ``Convolutional kernel
  networks,'' in \emph{NIPS}, 2014.

\bibitem{Jiu2016}
M.~Jiu and H.~Sahbi, ``Deep kernel map networks for image annotation,'' in
  \emph{ICASSP}, 2016.

\bibitem{Fan2008}
R.-E. Fan, K.-W. Chang, C.-J. Hsieh, X.-R. Wang, and C.-J. Lin, ``Liblinear: A
  library for large linear classification,'' \emph{JMLR}, vol.~9, pp.
  1871--1874, 2008.

\bibitem{Villegas2013}
M.~Villegas, R.~Paredes, and B.~Thomee, ``Overview of the imageclef 2013
  scalable concept image annotation subtask,'' in \emph{CLEF 2013 Evaluation
  Labs and Workshop}, 2013.

\bibitem{sahbi2013cnrs}
H.~Sahbi, ``Cnrs-telecom paristech at imageclef 2013 scalable concept image
  annotation task: Winning annotations with context dependent svms.'' in
  \emph{CLEF (Working Notes)}, 2013.

\bibitem{Duygulu2002}
P.~Duygulu, K.~Barnard, N.~de~Freitas, and D.~Forsyth, ``Object recognition as
  machine translation: Learning a lexicon for a fixed image vocabulary,'' in
  \emph{ECCV}, 2002.

\bibitem{Sahbi2015}
H.~Sahbi, ``Imageclef annotation with explicit context-aware kernel maps,''
  \emph{International Journal of Multimedia Information Retrieval}, vol.~4, pp.
  113--128, 2015.

\bibitem{li2011superpixel}
X.~Li and H.~Sahbi, ``Superpixel-based object class segmentation using
  conditional random fields,'' in \emph{Acoustics, Speech and Signal Processing
  (ICASSP), 2011 IEEE International Conference on}.\hskip 1em plus 0.5em minus
  0.4em\relax IEEE, 2011, pp. 1101--1104.

\bibitem{Guillaumin2009}
M.~Guillaumin, T.~Mensink, J.~Verbeek, and C.~Schmid, ``Tagprop: Discriminative
  metric learning in nearest neighbor models for image auto-annotation,'' in
  \emph{ICCV}, 2009, pp. 316--329.

\bibitem{Vo2012}
P.~Vo and H.~Sahbi, ``Transductive kernel map learning and its application to
  image annotation,'' in \emph{BMVC}, 2012, pp. 1--12.

\bibitem{ZhangDPR2012}
D.~Zhang, M.~Islam, and G.~Lu, ``A review on automatic image annotation
  techniques,'' \emph{Pattern Recognition}, vol.~45, 2012.

\bibitem{Murthyicmr2015}
V.~N. Murthy, S.~Maji, and R.~Manmatha, ``Automatic image annotation using deep
  learning representations,'' in \emph{International Conference on Multimedia
  Retrieval}, 2015, p. 603–606.

\bibitem{Jiutip2017}
M.~Jiu and H.~Sahbi, ``Nonlinear deep kernel learning for image annotation,''
  \emph{IEEE Transactions on Image Processing}, vol. 26(4), 2017.

\bibitem{HeZhangCVPR2016}
K.~He, X.~Zhang, S.~Ren, and J.~Sun, ``Deep residual learning for image
  recognition,'' in \emph{CVPR}, 2016, p. 770–778.

\end{thebibliography}

\end{document}